\documentclass[lettersize,journal]{IEEEtran}
\usepackage{amsmath,amsfonts}
\usepackage{algorithmic}
\usepackage{algorithm}
\usepackage{array}
\usepackage[caption=false,font=normalsize,labelfont=sf,textfont=sf]{subfig}
\usepackage{textcomp}
\usepackage{stfloats}
\usepackage{url}
\usepackage{verbatim}
\usepackage{graphicx}
\usepackage{cite}
\usepackage{times}  % DO NOT CHANGE THIS
\usepackage{helvet}  % DO NOT CHANGE THIS
\usepackage{courier}  % DO NOT CHANGE THIS
\usepackage{multirow}

\usepackage{amsmath}
\usepackage{amssymb}
\usepackage{xcolor}
\usepackage{amsthm}
\usepackage{enumitem}

\newtheorem{theorem}{Theorem}%[section]
\newtheorem{lemma}{Lemma}%[section]
\newtheorem{remark}{Remark}%[section]

\begin{document}

%\title{Suppressing Range of Adaptive Stepsizes in AdaBelief Improves Generalization of DNN Tasks}
\title{On Suppressing Range of Adaptive Stepsizes of Adam to Improve Generalisation Performance}

\author{Guoqiang~Zhang
        % <-this % stops a space
\thanks{\IEEEcompsocthanksitem G.~Zhang is with the School of Electrical and Data Engineering, University of Technology, Sydney, Australia. Email:{guoqiang.zhang@uts.edu.au}}

%\thanks{ K. Niwa is with both Communication Science Laboratories and Computer and Data Science Laboratories, Nippon Telegraph and Telephone Corporation (NTT). Email: {kenta.niwa.bk@hco.ntt.co.jp}} 

%\thanks{ W.~Bastiaan~Kleijn is with Victory University of Wellington. Email:{bastiaan.kleijn@ecs.vuw.ac.nz }}

}



\maketitle

\begin{abstract}
A number of recent adaptive optimizers improve the generalisation performance of Adam by essentially reducing the variance of adaptive stepsizes to get closer to SGD with momentum. Following the above motivation, we  suppress the range of the adaptive stepsizes of Adam by exploiting the layerwise gradient statistics.
In particular, at each iteration, we propose to perform three consecutive operations on the second momentum $\boldsymbol{v}_t$ before using it to update a DNN model: (1): down-scaling, (2): $\epsilon$-embedding, and (3): down-translating. The resulting algorithm is referred to as SET-Adam, where SET is a brief notation of the three operations. The down-scaling operation on $\boldsymbol{v}_t$ is performed layerwise by making use of the angles between the layerwise subvectors of $\boldsymbol{v}_t$ and the corresponding all-one subvectors. Extensive experimental results show that SET-Adam outperforms eight adaptive optimizers when training transformers and LSTMs for NLP, and VGG and ResNet for image classification over CIAF10 and CIFAR100 while matching the best performance of the eight adaptive methods when training WGAN-GP models for image generation tasks. Furthermore, SET-Adam produces higher validation accuracies than Adam and AdaBelief for training ResNet18 over ImageNet.  
\end{abstract}


%%%%%%%%%%%%%%%%%%%%%%%%%%%%%%%%%%%%%%%%%%%
\vspace{-2mm}
\section{Introduction}
\label{sec:intro}
\vspace{-1mm}

In the last decade, stochastic gradient descent (SGD) and its variants have been widely applied in deep learning \cite{Lecun15nature, Transformer17, Silver16GoGame, Chen20WaveGrad} due to their simplicity and effectiveness.  In the literature, SGD with momentum \cite{Sutskever13NAG, Polyak64CM}) dominates over other optimizers for image classification tasks \cite{He15ResNet, Wilson17AdamNegative}. Suppose the objective function $f({\boldsymbol{\theta}}):\boldsymbol{\theta}\in \mathbb{R}^d$ of a DNN model is differentiable. Its update expression for minimising $f(\boldsymbol{\theta})$ can be represented as
\begin{align}
\boldsymbol{m}_{t} & = \beta_t \boldsymbol{m}_{t-1} + \boldsymbol{g}_t  \label{equ:SGDM1}\\
\boldsymbol{\theta}_{t} &= \boldsymbol{\theta}_{t-1} - \eta_t \boldsymbol{m}_{t},  \label{equ:SGDM2}
\end{align}
where $\boldsymbol{g}_t = \nabla f(\boldsymbol{\theta}_{t-1})$ is the gradient at ${\boldsymbol{\theta}}_{t}$,  and $\eta_t$ is the common stepsize for all the coordinates of ${\boldsymbol{\theta}}$. %In practice, the above method is often combined with a certain step-size scheduling method for $\eta_t$ when training DNNs.

\begin{algorithm}[t!]
   \caption{\small SET-Adam: suppressing  range of adaptive stepsizes of Adam by performing three consecutive operations on $\{\boldsymbol{v}_t\}$}
   \label{alg:Adam}
\begin{algorithmic}[1]
   \STATE {\small {\bfseries Input:} $\beta_1$, $\beta_2$,  $\eta_t$, $\epsilon > 0$, $\tau=0.5$, $\boldsymbol{1}_l=[1,1,\ldots, 1]^T$  }
   \STATE {\small {\bfseries  Init.:} $\boldsymbol{\theta}_0\hspace{-0.5mm}\in\hspace{-0.5mm} \mathbb{R}^d$,  $\boldsymbol{m}_0 \hspace{-0.5mm}=\hspace{-0.5mm} 0$, $\boldsymbol{v}_{0}=0  \in \mathbb{R}^d$ }
   \FOR{\small $t=1, 2, \ldots, T$}
   \STATE \hspace{-0mm}{\small  $\boldsymbol{g}_t \leftarrow \nabla f({\boldsymbol{\theta}}_{t-1}) $  }
   \STATE \hspace{-0mm}{\small $\boldsymbol{m}_{t} \leftarrow \beta_1 \boldsymbol{m}_{t-1}  + (1-\beta_1) \boldsymbol{g}_t$ }
   \STATE \hspace{-0mm}{\small $\boldsymbol{v}_{t} \leftarrow \beta_2 \boldsymbol{v}_{t-1}  + (1-\beta_2) \boldsymbol{g}_t^2$ }
   \FOR{\small $l=1,\ldots, L$}
   \STATE {\small $\tilde{\boldsymbol{v}}_{l,t}=\boldsymbol{v}_{l,t}\cos^2(\angle \boldsymbol{v}_{l,t}\boldsymbol{1}_l )\qquad\quad\quad \textcolor{blue}{\textbf{[down-scaling]}}$}
   \STATE {\small ${\boldsymbol{w}}_{l,t}=\sqrt{\tilde{\boldsymbol{v}}_{l,t}/(1-\beta_2^t)\textcolor{blue}{+\epsilon}}\quad\quad\quad\; \textcolor{blue}{\textbf{[}\epsilon \textbf{-embedding]}}$}   
   \STATE {\small $\tilde{\boldsymbol{w}}_{l,t}={\boldsymbol{w}}_{l,t}-\tau \left(\min_{i=1}^{d_l} {\boldsymbol{w}}_{l,t}[i]\right) \; \textcolor{blue}{[\textbf{down-translating]}}$}      
   \ENDFOR
   \STATE \hspace{-0mm}{\small $  \tilde{\boldsymbol{m}}_{t} \hspace{-0.6mm}\leftarrow \frac{\boldsymbol{m}_{t}}{1-\beta_1^{t}} $ } 
  \STATE {\small  $\boldsymbol{\theta}_{t} \hspace{-0.6mm}\leftarrow \boldsymbol{\theta}_{t-1} -\frac{\eta_t\tilde{\boldsymbol{m}}_{t} }{\tilde{\boldsymbol{w}}_{t} } $}
   \ENDFOR 
   \STATE {\bfseries Output:} {\small $\boldsymbol{\theta}_{T}$  }
\end{algorithmic}
\end{algorithm}

To bring flexibility to SGD with momentum, an active research trend is to introduce elementwise adaptive stepsizes for all the coordinates of $\boldsymbol{m}_{t}$ in (\ref{equ:SGDM2}), referred to as \emph{adaptive optimization} \cite{Duchi11AdaGrad, Tieleman12RMSProp, Kingma17}. 
%\begin{figure}[t!]
%\centering
%\includegraphics[width=85mm]{Aida_K2_ALR_std_overall.eps}
%\includegraphics[width=140mm]{Aida_3stage_overall_std_mean.eps}
%\vspace*{-0.0cm}
%\caption{\footnotesize Layerwise standard deviations of elementwise stepsizes and layerwise average stepsizes for the top 10 neural layers %\footnotemark  
%when training VGG11 over CIFAR10 for 200 epochs. The jumps at 100 and 160 epoch in the curves are due to the change of the common stepsize. Adam+ is obtained by  replacing $(\boldsymbol{m}_t-\boldsymbol{g}_t)^2$ with $\boldsymbol{g}_t^2$ in AdaBelief where the parameter $\epsilon$ are placed differently from the original Adam (see the location differences of $\epsilon$ between (\ref{equ:Adam3}) and (\ref{equ:st_AdaBelief})-(\ref{equ:xt_AdaBelief}). The curves for Aida tend to be much more compact than those for AdaBelief and Adam+. See Fig.~\ref{fig:Aida_mean_std_compare} for more a concise comparison.}
%\label{fig:Aida_std_mean_compare}
%\vspace*{-0.3cm}
%\end{figure}
In the literature, Adam \cite{Kingma17} is probably the most popular adaptive optimization method (e.g., \cite{Transformer17,liu2021Swin,Zhang2019Adam,Devlin18Bert}). Its update expression can be written as 
\vspace{-0mm}
\begin{align}
\hspace{-10mm}[\textbf{Adam}]\hspace{0mm} &\left\{\begin{array}{l}
\hspace{-2mm}\boldsymbol{m}_{t}  =\beta_1\boldsymbol{m}_{t-1} + (1-\beta_1) \boldsymbol{g}_t      \qquad \qquad  \qquad  \quad \\
\hspace{-2mm}\boldsymbol{v}_{t}  = \beta_2 \boldsymbol{v}_{t-1} + (1-\beta_2) \boldsymbol{g}_t^2  \qquad  \qquad \qquad \quad   \\
%{x}_{t} &= {x}_{t-1} - \left\{\begin{array}{c} \eta \frac{\sqrt{1-\beta_2^t}}{1-\beta_1^t} \frac{{m}_{t}}{\sqrt{{v}_{t}+\epsilon}}   \\ 
% \eta \frac{\sqrt{1-\beta_2^t}}{1-\beta_1^t} \frac{{m}_{t}}{\sqrt{{v}_{t}}+\epsilon} 
%  \end{array}\right.  \label{equ:Adam3}
\hspace{-2mm} \boldsymbol{\theta}_{t}  = \hspace{-0.8mm} \boldsymbol{\theta}_{t-1}   \hspace{-0.8mm} - \hspace{-0.8mm}  \eta_t \frac{1}{1-\beta_1^{t}}  \frac{ \boldsymbol{m}_{t} }{\sqrt{\boldsymbol{v}_{t}/(1-\beta_2^{t})}+\epsilon}
\end{array}\right.\hspace{-20mm},
\label{equ:Adam}
\end{align}
where $\boldsymbol{g}_t=f(\boldsymbol{\theta}_{t-1})$, $0<\beta_1, \beta_2<1$, and $\epsilon>0$. The two vector operations $(\cdot)^2$ and $\cdot/\cdot$ are performed in an elementwise manner. The two exponential moving averages (EMAs) $\boldsymbol{m}_{t}$ and $\boldsymbol{v}_{t}$ are referred to as the first and second momentum. The two quantities $1-\beta_1^{t}$ and ${1-\beta_2^{t}}$ are introduced to compensate for the estimation bias in $\boldsymbol{m}_{t}$ and $\boldsymbol{v}_{t}$, respectively. $\eta_t$ is the common stepsize while 
$\boldsymbol{\alpha}_t=1/\left(\sqrt{\boldsymbol{v}_t/(1-\beta_2^{t})}+\epsilon\right)\in \mathbb{R}^d$
represents the elementwise adaptive stepsizes. %and the constant $\epsilon>0$ is introduced to avoid division by zero. $m_t$ and $v_t$ are the two moments of gradients and squared-gradients over iterations, respectively.  The two quantities $1-\beta_1^{t}$ and $\sqrt{1-\beta_2^{t}}$  in (\ref{equ:Adam3}) are introduced as bias-correction to compensate for the estimation bias in $m_{t}$ and $v_{t}$, respectively.  %Note that the update for $x_{t+1}$ may have different forms of expression depending on how the parameter $\epsilon$ participates in computation of the individual learning rates. As indicated by (\ref{equ:Adam3_1}) for example, the update direction $\frac{|{m}_{t+1}|\odot \textrm{sign}(m_{t+1}) }{\sqrt{{v}_{t+1}+\epsilon}}$ for $x_{t+1}$ can be viewed as the sign vector $\textrm{sign}(m_{t+1})$ multiplied element-wise by the magnitude vector $\frac{|{m}_{t+1}|}{\sqrt{{v}_{t+1}+\epsilon}}$.  

\footnotetext{Considering Adam at iteration $t$, the layerwise average of adaptive stepsizes for the $l$th layer of VGG11 is computed as $\frac{1}{d_l}\sum_{i=1}^{d_l}\boldsymbol{\alpha}_{l,t}[i]=\frac{1}{d_l}\sum_{i=1}^{d_l}1/(\sqrt{\boldsymbol{v}_{l,t}[i]/(1-\beta_2^t)}+\epsilon)$, where $\boldsymbol{v}_{l,t}\in \mathbb{R}^{d_l}$ is the subvector of $\boldsymbol{v}_{t}\in\mathbb{R}^d$ for the $l$th layer. }

Due to the great success of Adam in training DNNs, various extensions of Adam have been proposed, including AdamW \cite{Loshchilov19AdamW}, NAdam \cite{Dozat16NAdam}, Yogi \cite{Zaheer18Yogi}, MSVAG  \cite{Balles17MSVAG}, Fromage \cite{Bernstein20Fromage}. It is worth noting that the following three algorithms extend Adam by reducing the variance of the adaptive stepsizes in different manner.  In \cite{Liu19RAdam}, the authors suggested multiplying  $\boldsymbol{m}_t$ by a rectified scalar when computing $\boldsymbol{\theta}_t$ in (\ref{equ:Adam}) when the variance is large, which is referred to as RAdam. Empirical evidence shows that RAdam produces better generalization. The AdaBound method of \cite{Luo19AdaBound} is designed to avoid extremely large and small adaptive stepsizes of Adam, which has a similar effect as RAdam. In practice, AdaBound works as an adaptive method at the beginning of the training process and gradually transforms to SGD with momentum, where all the adaptive stepsizes tend to converge to a single value. The AdaBelief method of \cite{Zhuang20Adabelief} extends Adam by tracking the EMA of the squared prediction error $(\boldsymbol{m}_t-\boldsymbol{g}_t)^2$ instead of $\boldsymbol{g}_t^2$ when computing the second momentum $\boldsymbol{v}_t$. In general, $\boldsymbol{m}_t$ can be viewed as a reliable prediction of $\boldsymbol{g}_t$, making the mean and variance of $(\boldsymbol{m}_t-\boldsymbol{g}_t)^2$ across all the coordinates smaller than those of $\boldsymbol{g}_t^2$. As a result, the variance of the adaptive stepsizes in AdaBelief would be smaller than that in Adam. Conceptually speaking, the above three methods aim to reduce the range of the adaptive stepsizes of Adam to mimic the convergence behavior of SGD with momentum to a certain extent. %See also \cite{Wilson17AdamNegative} for a discussion on the relationship between the lack of generalization performance of adaptive methods and extreme adaptive stepsizes.  

Based on our observation that the gradient statistics in Adam are generally heterogeneous across different neural layers  (see Fig.~\ref{fig:SETAdam_mean_compare}-\ref{fig:SETAdam_std_compare}), we consider suppressing the range of adaptive stepsizes of Adam by using layerwise gradient statistics. Suppose at iteration $t$, $\boldsymbol{v}_t$ is the obtained second momentum by following (\ref{equ:Adam}). A sequence of three operations are performed on $\boldsymbol{v}_t$ (see Alg.~1 for details) before using it in the update of the DNN model $\boldsymbol{\theta}$, namely, (1): down-scaling, (2): $\epsilon$-embedding, and (3): down-translating. The resulting algorithm is referred to as SET-Adam, where  ``SET'' is a short notation of the three operations.  Fig.~\ref{fig:SETAdam_mean_compare} and \ref{fig:SETAdam_std_compare} demonstrates that SET-Adam indeed produces a more compact range of adaptive stepsizes than Adam  for training VGG11 over CIFAR10. At the end of the training process, the 11 layerwise average stepsizes in Fig.~\ref{fig:SETAdam_mean_compare}  do not converge to a single value, indicating the adaptability of SET-Adam.

\begin{figure}[t!]
\centering
\includegraphics[width=75mm]{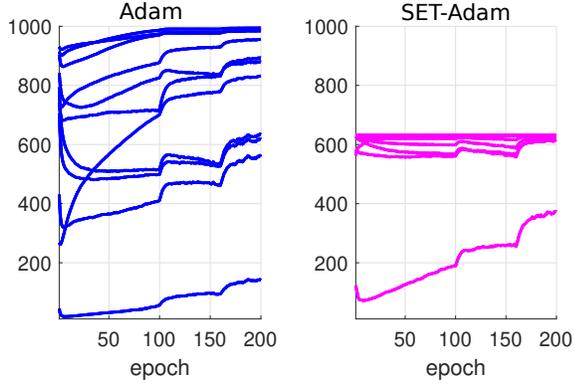}
\vspace*{-0.0cm}
\caption{\small Comparison of layerwise average of adaptive stepsizes for the 11 neural layers of VGG11 by training over CIFAR10 for 200 epochs. See Appendix~\ref{appendix:fig_setup} for the parameter setups of the two methods, where \textcolor{blue}{the optimal parameter $\epsilon$ for Adam was selected from a discrete set to give the best validation performance}. The jumps in the curves at 100 and 160 epochs are due to the change in the common stepsize. SET-Adam has a much more compact range of layerwise average stepsizes than Adam.   }
\label{fig:SETAdam_mean_compare}
\vspace*{-0.0cm}
\end{figure}

To summarize, we make four contributions with SET-Adam. \vspace{-4.5mm}
\begin{enumerate}[label=(\alph*)]
    \item We propose to down-scale the layerwise subvectors $\{\boldsymbol{v}_{l,t}\}_{l=1}^L$ of $\boldsymbol{v}_t$ by exploiting the angles between the subvectors $\{\boldsymbol{v}_{l,t}\}_{l=1}^L$ and the corresponding all-one vectors $\{\boldsymbol{1}_l\}_{l=1}^L$ of the same dimensions. Specifically, the subvectors of  $\boldsymbol{v}_t$ of which the angles are relatively large are down-scaled to a large extent. On the other hand, a small angle indicates that the associated subvector $\boldsymbol{v}_{l,t}$ has roughly the same value across all its coordinates, implying this particular neural layer has great similarity to the effect of SGD with momentum at the current iteration. In this case, it is reasonable to make the down-scaling negligible. It can be shown that the above operation is able to suppress the range of adaptive stepsizes compared to Adam.  \vspace{-0mm}
    \item Differently from (\ref{equ:Adam}) of Adam, the $\epsilon$-embedding operation puts $\epsilon$ inside the sqrt operation (see Alg.~1). We are aware that the AdaBelief optimizer also utilizes the above $\epsilon$-embedding operation but without a proper motivation (see Appendix~\ref{appendix:AdaBelief_embedding} for mathematical verification). We show via a Taylor expansion that putting $\epsilon$ inside the sqrt operation essentially suppresses the range of the adaptive stepsizes.  \vspace{-0mm}
    \item As indicated in Alg.~1, the down-translating operation subtracts a scalar value from $\boldsymbol{w}_{l,t}$ for the $l$th neural layer, where the scalar is computed as a function of $\boldsymbol{w}_{l,t}$.  It is designed to uplift the resulting adaptive stepsizes to avoid extreme small ones, which is inspired by the design of the AdaBound optimizer. It is found that the setup $\tau=0.5$ for controlling the strength of the down-translating works well in practice, and there is no need to tune the parameter.  
    \item Extensive experimental results show that SET-Adam yields considerably better performance than eight adaptive optimizers for training transformer \cite{Transformer17} and LSTM \cite{HochSchm97LSTM} models in natural language processing (NLP) tasks, and VGG11 \cite{Simonyan14VGG} and ResNet34 in image classification tasks over CIFAR10 and CIFAR100. It is also found that SET-Adam  matches the best performance of the eight methods when training WGAN-GP models in image generation tasks. Lastly, SET-Adam outperforms Adam and AdaBelief when training ResNet18 on the large ImageNet dataset. The computational complexity of SET-Adam was evaluated for training VGG11 and ResNet34. The results show that SET-Adam consumes an additional $12\%-20\%$ time per epoch compared to Adam. \vspace{-0mm}
\end{enumerate}

\begin{figure}[t!]
\centering
\includegraphics[width=75mm]{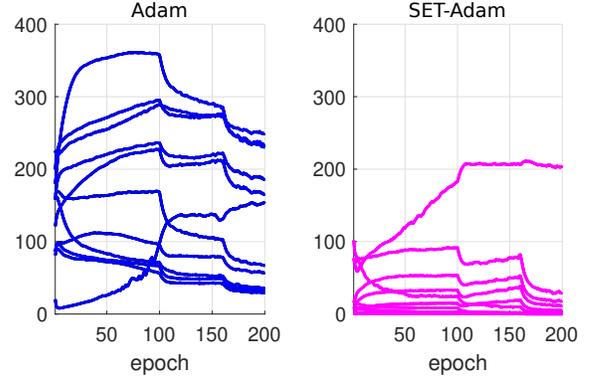}
\vspace*{-0.0cm}
\caption{\small{Comparison of layerwise standard deviations (stds) of adaptive stepsizes for the 11 neural layers by training VGG11 over CIFAR10 for 200 epochs. SET-Adam has much smaller layerwise stds than Adam for ten out of eleven nueral layers.    } }
\label{fig:SETAdam_std_compare}
\vspace*{-0.5cm}
\end{figure}

% 

 %motivated by the observations in the first step, we propose to track the EMA of the \emph{layer-wise mean} value of square-gradient or its variants and then use them to compute the layer-wise adaptive step-sizes accordingly. We refer to the new layer-wise adaptive optimization framework as Aida.  Alg.~2 summarises the update procedure of Aida-grad which tracks the statistics of layer-wise squared-gradient over iterations (see Alg.~4 in Appendix~\ref{apsecA_belief} for Aida-belief as an extension of AdaBelief). 
%Convergence results are provided for Aida which follows almost the same analysis in \cite{Zhuang20Adabelief} for AdaBelief. 

\noindent\textbf{Notations}: We use small bold letters to denote vectors.  The $l_2$ and $l_{\infty}$ norms of a vector $\boldsymbol{y}$ are denoted as $\|\boldsymbol{y}\|_2$ and $\|\boldsymbol{y}\|_{\infty}$, respectively. Given an $L$-layer DNN model $\boldsymbol{\theta}$ of dimension $d$, we use $\boldsymbol{\theta}_l$ of dimension $d_l$ to denote the subvector of $\boldsymbol{\theta}$ for the $l$th layer. Thus, there is $\sum_{l=1}^L d_l=d$. The $i$th element of $\boldsymbol{\theta}_l$ is represented by $\boldsymbol{\theta}_l[i]$. The notation $[L]$ stands for the set $[L] = \{1,\ldots, L\}$. Finally, the angle between two vectors $\boldsymbol{y}$ and $\boldsymbol{x}$ of the same dimension is denoted by $\angle \boldsymbol{x}\boldsymbol{y}$. 

\vspace{-0mm}
\section{Algorithmic Design of SET-Adam }
\label{sec:alg}
\vspace{-0mm}

In this section, we explain in detail the three operations introduced in SET-Adam sequentially for processing the second momentum $\boldsymbol{v}_t$ at iteration $t$ (see Alg.~1). We point out that the first two operations \emph{down-scaling} and \emph{$\epsilon$-embedding} suppress the range of adaptive stepsizes of Adam while the last operation \emph{down-translating} uplifts the resulting adaptive stepsizes to avoid extreme small ones. A convex convergence analysis for SET-Adam is presented at the end of the section.    

%In this section, we first  motivate a general strategy that could suppress the range of adaptive stepsizes of Adam and AdaBelief. After that, we present the update expressions of Aida and AidaGrad.  Convex convergence analysis is presented in the end.       

%In this subsection, we first explain in detail why our new method Aida is ought to track a quantity of $\boldsymbol{m}_{l,t}$ and $\boldsymbol{g}_{l,t}$ for the $l$th layer of a DNN model with smaller statistic variance and average value over all the coordinates than $(\boldsymbol{m}_{l,t}-\boldsymbol{g}_{l,t})^2$ being employed in AdaBelief. After that, we present the update expressions of Aida.  Convex convergence analysis is presented in the end. 

\vspace{-1mm}
\subsection{Motivation of layerwise down-scaling operation}\vspace{-0mm}
%\noindent\textbf{Motivation}:
Motivated by the existing work \cite{Liu19RAdam,Luo19AdaBound,Zhuang20Adabelief}, we would like to reduce the range of the adaptive stepsizes of Adam to make the new optimizer get closer to SGD with momentum. To achieve the above goal, we consider processing the second momentum $\boldsymbol{v}_t$ of (\ref{equ:Adam}) in a layerwise manner before using it to update a DNN model $\boldsymbol{\theta}$. In general, the parameters within the same neural layer are functionally homogeneous when processing data from the layer below. It is likely that the gradients within the same layer follow a single distribution, making it natural to perform layerwise processing. On the other hand, the gradient statistics can be quite different across different neural layers due to progressive nonlinear processing from bottom to top layers. See Fig.~\ref{fig:SETAdam_mean_compare}-\ref{fig:SETAdam_std_compare} for the heterogeneous layerwise curves of Adam. Finally, it is computationally more efficient to perform layerwise processing than operating on the entire vector $\boldsymbol{v}_t$ due to the nature of back-propagation when training the model $\boldsymbol{\theta}$,

We now consider what kind of layerwise processing would be desirable to extend Adam. Firstly, we note that the parameter $\epsilon$ of (\ref{equ:Adam}) in Adam defines an upper bound on the adaptive stepsizes and is independent of neural layer and iteration indices. We use $\boldsymbol{v}_{l,t}$ and $\boldsymbol{\alpha}_{l,t}$ to denote the subvectors of $\boldsymbol{v}_{t}$ and $\boldsymbol{\alpha}_{t}$ for the $l$th neural layer, respectively. It is immediate that the upper bound\footnote{For  the case that $\epsilon$ is inside the sqrt operation, the uppper bound becomes 
$\frac{1}{\sqrt{\epsilon}}\geq \max_{l=1}^L\max_{i=1}^{d_l} (\boldsymbol{\alpha}_{l,t}[i] \hspace{-0.8mm}=\hspace{-0.8mm}  1/\sqrt{\boldsymbol{v}_{l,t}[i]/(1\hspace{-0.8mm}-\hspace{-0.8mm}\beta_2^t)+\hspace{-0.8mm}\epsilon}  )$. In this case, the motivation for down-scaling still holds. } is
\begin{align}
\frac{1}{\epsilon}\geq \max_{l=1}^L\max_{i=1}^{d_l} \Big(\boldsymbol{\alpha}_{l,t}[i] \hspace{-0.8mm}=\hspace{-0.8mm}  1/\Big(\sqrt{\boldsymbol{v}_{l,t}[i]/(1\hspace{-0.8mm}-\hspace{-0.8mm}\beta_2^t)}\hspace{-0.8mm}+\hspace{-0.8mm}\epsilon\Big)  \Big).
\label{equ:up_bound_Adam}
\end{align}
Suppose we replace each $\boldsymbol{v}_{l,t}$ in (\ref{equ:up_bound_Adam}) by $\gamma_{l,t}\boldsymbol{v}_{l,t}$, where $\gamma_{l,t}\in (0,1)$. If the scalars $\{\gamma_{l,t}\}_{l=1}^L$ are sufficiently small in the extreme case, all the adaptive stepsizes of the new method tend to approach the upper bound in (\ref{equ:up_bound_Adam}). As a result, the new method will have a smaller range of adaptive stepsizes than Adam either in a layerwise manner or globally. %Similarly, Adam can be extended by tracking the EMA of $(\beta_{l,t}\boldsymbol{g}_{l,t})^2$ for the $l$th layer.   

We propose to compute the scalar $\gamma_{l,t}$ for $\boldsymbol{v}_{l,t}$ as a function of the angle $\angle \boldsymbol{v}_{l,t}\boldsymbol{1}_{l}$ between the subvector $\boldsymbol{v}_{l,t}$ and the corresponding all-one vector $\boldsymbol{1}_l=[1,1,\ldots, 1]^T$. 
The angles $\{\angle \boldsymbol{v}_{l,t}\boldsymbol{1}_{l}\}_{l=1}^L$ reflect the variances of $\{\boldsymbol{v}_{l,t}\}_{l=1}^L$ across their respective coordinates to a certain extent, which are different across different layers.  
%This is motivated by the property that the  statistics of $\{\boldsymbol{v}_{l,t}\}_{l=1}^L$ in Adam are heterogeneous across different neural layers. %As an example, in Fig.~\ref{fig:SETAdam_std_compare}, the 11 layerwise standard deviation curves of the adaptive stepsizes in Adam are distributed in a wide range between 0 and 360.
By doing so, we implictly take into account the heterogeneity of layerwise gradient statistics instead of setting an iteration-dependent common value for $\gamma_{l,t}$ across all neural layers. 

On the other hand, RAdam and AdaBound in the literature do not take into account the heterogeneity of statistics in $\{\boldsymbol{v}_{l,t}\}_{l=1}^L$ accross different layers. This might be the reason that the empirical performance of SET-Adam is better than the above two optimizers as will be demonstrated in Section~\ref{sec:exp}.

%$K$ mutual vector projections starting from $(\boldsymbol{m}_{l,t},\boldsymbol{g}_{l,t})$ (see Fig.~\ref{fig:vector_projection} for demonstration). In practice, it is found that $K=2$ is sufficient to produce small scalars $(\gamma_{l,t},\beta_{l,t})$, leading to a smaller range of adaptive stepsizes than those of AdaBelief and Adam. The parameter $K$ of Aida in Fig.~\ref{fig:Aida_mean_compare}-\ref{fig:Aida_std_compare} was set to $K\in \{1,2\}$. In the following, the update expressions of Aida are presented in detail. 

\subsection{Design of layerwise down-scaling operation}
Consider the $l$th layer of a DNN model at iteration $t$. We down-scale $\boldsymbol{v}_{l,t}$ by a scalar $ \gamma_{l,t}$, given by\footnote{We have also considered the high-order case $\gamma_{l,t}=\cos^n(\angle \boldsymbol{v}_{l,t}\boldsymbol{1}_l)$, $n=4$. It is empirically found that $n=4$ makes the down-scaling operation too aggressive while at the same time incurring additional computation cost. We therefore ignore the high-order cases in the main paper.  }
\begin{align}
    \tilde{\boldsymbol{v}}_{l,t}=\gamma_{l,t}\boldsymbol{v}_{l,t}=
    \cos^2(\angle\boldsymbol{v}_{l,t}\boldsymbol{1}_{l})\boldsymbol{v}_{l,t} \label{equ:downScaling}
\end{align}
As illustrated in Fig.~\ref{fig:vector_projection}, the quantity $ \tilde{\boldsymbol{v}}_{l,t}$ can be viewed as performing two vector projections between  $\boldsymbol{v}_{l,t}$ and $\boldsymbol{1}_{l}$. That is, we first project $\boldsymbol{v}_{l,t}$ onto the vector direction $\boldsymbol{1}_{l}/\|\boldsymbol{1}_{l}\|$, and then project back the obtained vector onto the vector direction $\boldsymbol{v}_{l,t}/\|\boldsymbol{v}_{l,t}\|$.  

Our motivation for choosing the all-one vector $\boldsymbol{1}_l$ as a reference for the angle computation is that we would like to push the adaptive stepsizes of Adam to be much more compact so that the new optimizer has a similar effect as that of SGD with momentum. With the expression (\ref{equ:downScaling}), it is natural that those layerwise subvectors of $\boldsymbol{v}_t$ with large angles $\{\angle\boldsymbol{v}_{l,t}\boldsymbol{1}_{l}\}$ are down-scaled to a large extent. On the contrary, a small angle implies that all the elements in the associated subvector $\boldsymbol{v}_{l,t}$ are roughly the same. In this case, the scalar $\gamma_{l,t}$ is close to 1, suggesting that the down-scaling operation is negligible as desired.

\begin{figure}[t!]
\centering\includegraphics[width=75mm]{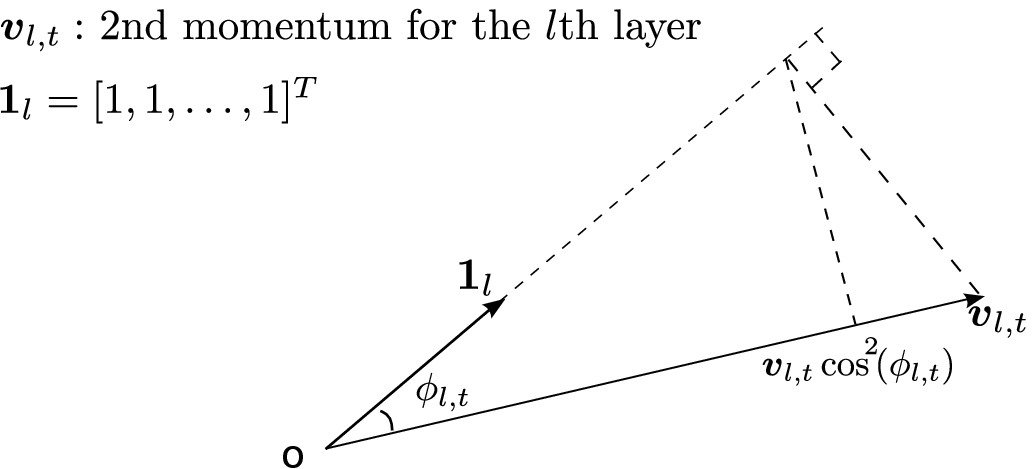}
\vspace*{-0.0cm}
\caption{\small{ Demonstration of the \textcolor{blue}{down-scaling} operation in SET-Adam. The vector $\boldsymbol{1}_l$ is of the same dimension as $\boldsymbol{v}_{l,t}$} }
\label{fig:vector_projection}
\vspace*{-0.45cm}
\end{figure}

%Next, we consider the geometric properties of the set of projected vectors.  It is not difficult to show that after projection, the resulting vectors have either shorter or equal length in comparison to the original vectors: \begin{align}
%\|\boldsymbol{m}_{l,t}^{(k)}\|_2 \leq %\|\boldsymbol{m}_{l,t}^{k-1}\|_2
%\quad \textrm{and} \quad \|\boldsymbol{g}_{l,t}^{(k)}\|_2 \leq \|\boldsymbol{g}_{l,t}^{(k-1)}\|_2. \label{equ:proj_ineq} \end{align}
%Using the fact that mutual projections of two vectors do not change the angle, we then have
%\begin{align}
%    \| \boldsymbol{m}_{l,t}^{(k)} - \boldsymbol{g}_{l,t}^{(k)} \|_2 \leq \| \boldsymbol{m}_{l,t}^{(k-1)} - \boldsymbol{g}_{l,t}^{(k-1)} \|_2, 
%    \label{equ:proj_norm_ineq}
%\end{align}
%where the equality holds if $\boldsymbol{m}_{l,t}$ and $\boldsymbol{g}_{l,t}$ are on the same line and $\xi$ can be ignored in  (\ref{equ:m_g_proj}). 

For the extreme case that each neural layer has only one parameter (i.e., $\boldsymbol{v}_{l,t}\in \mathbb{R}$, $\forall l\in [L]$), it is immediate that 
$\gamma_{l,t}=\cos^2(\angle\boldsymbol{v}_{l,t}\boldsymbol{1}_{l})=1$. That is, the down-scaling operation has no effect, which is safe from the viewpoint of implementation stability.

\vspace{-2mm}
\subsection{$\epsilon$-embedding for suppressing range of adaptive stepsizes }\vspace{-1mm}

It is known from literature that the  $\epsilon$ parameter is originally introduced in adaptive optimizers to avoid division by zero. In this subsection, we show that the placement of the $\epsilon$ parameter in the update expressions has a significant impact on the adaptive stepsizes. In particular, we will show in the following that by a Taylor approximation that putting the $\epsilon$ parameter inside the sqrt operation of (\ref{equ:Adam}) in Adam helps to suppress the range of adaptive stepsizes. 

Suppose Adam$^\star$ is obtained by putting  $\epsilon$ inside the sqrt operation in (\ref{equ:Adam}) (see Appendix~\ref{appendix:adamplus} for Adam$^\star$). We now study the modified adaptive stepsizes in Adam$^{\star}$: $\tilde{\boldsymbol{\alpha}}_t= 1/\sqrt{\boldsymbol{v}_t/(1-\beta_2^t)+\epsilon}$, which can be approximated to be
\begin{align}
\tilde{\boldsymbol{\alpha}}_t&= 1/\sqrt{{\boldsymbol{v}}_t/(1-\beta_2^t)+\epsilon} \label{equ:epsilon_effect0} \\
&\approx  \frac{1}{\underbrace{\sqrt{{\boldsymbol{v}}_t/(1-\beta_2^t)}}_{\textrm{1st term}}+\underbrace{\frac{1}{2\sqrt{{\boldsymbol{v}}_t/(1-\beta_2^t)}}\epsilon}_{\textrm{2nd term}} },\label{equ:epsilon_effect}
\end{align}
where the Taylor approximation is applied to a function $h(x)=\sqrt{\boldsymbol{a}+x}$ around $x=0$, where $x=\epsilon$ and $\boldsymbol{a}={\boldsymbol{v}}_t/(1-\beta_2^t)$.

\begin{figure}[t!]
\centering
\includegraphics[width=80mm]
{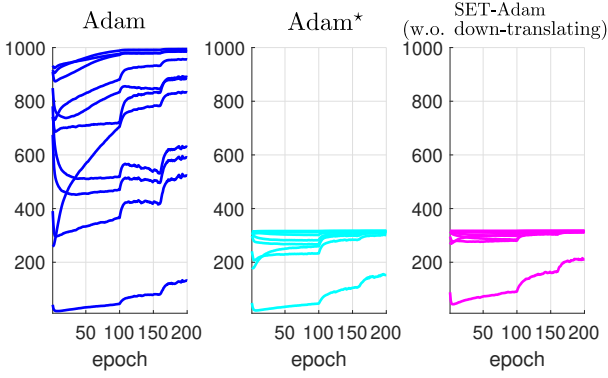}
%{AdamPlus_compare_11layer_pro_2nd.eps}
\vspace*{-0.0cm}
\caption{\small{Comparison of layerwise average of adaptive stepsizes for the 11 neural layers by training VGG11 over CIFAR10 for 200 epochs.  \textcolor{blue}{For the plot of SET-Adam, the down-translating operation is ignored and only the first two operations are included.} See Appendix~\ref{appendix:fig_setup} for the algorithmic parameter setups. } }
\label{fig:AdamPlus}
\vspace*{-0.3cm}
\end{figure}

%We now investigate (\ref{equ:epsilon_effect}). Generally speaking, large elements of $\hat{\boldsymbol{s}}_t$ would lead to small adaptive stepsizes while small elements lead to large adaptive stepsizes due to the inverse operation $1/(\cdot)$. It is clear from (\ref{equ:epsilon_effect}) that for large elements of $\hat{\boldsymbol{s}}_t$, the second term in the denominator of (\ref{equ:epsilon_effect}) is relatively small, and therefore have little effect on the resulting small adaptive stepsizes. In contrast, for small elements of $\hat{\boldsymbol{s}}_t$, the second term is relatively large, thus avoiding extremely large adaptive stepsizes. To briefly summarize, the inclusion of $\epsilon$ in computation of $\boldsymbol{s}_t$ suppresses the range of adaptive stepsizes in AdaBelief  by avoiding extremely large stepsizes. 

Next we investigate (\ref{equ:epsilon_effect}). Generally speaking,  small elements of ${\boldsymbol{v}}_t$ lead to large adaptive stepsizes while large elements lead to small adaptive stepsizes due to the inverse operation $1/(\cdot)$. It is clear from (\ref{equ:epsilon_effect}) that for small elements of ${\boldsymbol{v}}_t$, the second term in the denominator is relatively large, implicitly penalizing large stepsizes. Furthermore, (\ref{equ:epsilon_effect0}) indicates that those large stepsizes are upper-bounded by the quantity $1/\sqrt{\epsilon}$.  In contrast, for large elements of ${\boldsymbol{v}}_t$, the second term is relatively small, thus reducing the occurrence of extremely small adaptive stepsizes. In short, moving $\epsilon$ from outside of the sqrt operation to inside helps to suppresses the range of adaptive stepsizes in Adam.

Similarly, employment of the $\epsilon$-embedding in SET-Adam would also help to suppress the range of adaptive stepsizes. Mathematically, applying the $\epsilon$-embedding on $\tilde{\boldsymbol{v}}_{l,t}$ of (\ref{equ:downScaling}) leads to 
\begin{align}
 \boldsymbol{w}_{l,t}&=\sqrt{\tilde{\boldsymbol{v}}_{l,t}/(1-\beta_2^t)+\epsilon} \nonumber\\ 
 &=   \sqrt{{\cos^2(\angle \boldsymbol{v}_{l,t}\boldsymbol{1}_l)\boldsymbol{v}}_{l,t}/(1-\beta_2^t)+\epsilon}. \label{equ:embedding}
\end{align}
 As we mentioned earlier, AdaBelief also utilizes $\epsilon$-embedding in the update expressions (see Appendix~\ref{appendix:AdaBelief_embedding} for verification), which might be another reason AdaBelief outperforms Adam.    

Fig.~\ref{fig:AdamPlus} demonstrates that   Adam$^\star$ indeed has a more compact range of adaptive stepsizes than Adam due to the proper placement of the $\epsilon$ parameter. SET-Adam without the down-translating operation further suppresses the range of adaptive stepsizes of Adam$^\star$ due to the additional down-scaling operation. 
At epoch 200, the eleven layerwise average stepsizes in Adam are distributed in [190,1000] while ten out of eleven layerwise average stepsizes in SET-Adam are close to a single value of 320. 
 %The respective validation accuracy for each method is also included in Fig. ~\ref{fig:AdaBeliefMinus}, which is obtained by three independent experimental repetitions.  It shows that smaller range of adaptive stepsizes in AdaBelief leads to higher validation accuracy, which is consistent with the findings of \cite{Liu19RAdam,Luo19AdaBound}.  %As are presented in Algorithm 1, our two new methods \emph{Aida} and \emph{AidaGrad} also include $\epsilon$ in computa $\emph{AidaGrad}$ .      

\vspace{-2mm}
\subsection{Down-translating for avoiding extreme small adaptive stepsizes}
\vspace{-1mm}

We notice that due to the $\epsilon$-embedding operation in (\ref{equ:embedding}), the elements in the subvector $\boldsymbol{w}_{l,t}$ are lower-bounded by the positive scalar $\sqrt{\epsilon}$, i.e., $\boldsymbol{w}_{l,t}[i]\geq \sqrt{\epsilon}$, $\forall i$. We propose to further subtract a scalar from $\boldsymbol{w}_{l,t}$ to avoid extreme small adaptive stepsizes in SET-Adam, which is given by
\begin{align}
 \tilde{\boldsymbol{w}}_{l,t} = \boldsymbol{w}_{l,t} -\tau \left(\min_{i=1}^{d_l} \boldsymbol{w}_{l,t}[i]\right), \label{equ:translating}   
\end{align}
where $0\leq\tau<1$ is imposed to ensure that all the elements of $\tilde{\boldsymbol{w}}_{l,t}$ are positive. It is not difficult to show that
\begin{align}
  (1-\tau)\sqrt{\epsilon} \leq \tilde{\boldsymbol{w}}_{l,t}[i]\quad \forall i=1,\ldots, d_l. \label{equ:w_lower_bound}    
\end{align}
That is, the subvector $\tilde{\boldsymbol{w}}_{l,t}$ is lower-bounded by $(1-\tau)\sqrt{\epsilon}$. Once all the subvectors $\{\tilde{\boldsymbol{w}}_{l,t}\}_{l=1}^L $ are obtained from (\ref{equ:translating}), the DNN model $\boldsymbol{\theta}$ can be updated in the form of
\begin{align}
\boldsymbol{\theta}^{t} = \boldsymbol{\theta}^{t-1} -\eta\tilde{\boldsymbol{m}}_t/\tilde{\boldsymbol{w}}_t,\nonumber    
\end{align}
where the adaptive stepsizes are $\boldsymbol{\alpha}_t=1/\tilde{\boldsymbol{w}}_t$. 

By inspection of the behaviors of SET-Adam in Figs.~\ref{fig:SETAdam_mean_compare} and \ref{fig:AdamPlus}, it can be seen that the down-translating operation manages to uplift the resulting adaptive stepsizes of SET-Adam in Fig.~\ref{fig:SETAdam_mean_compare}, which implicitly avoids extreme small adaptive stepsizes as expected.    

%provide empirical evidence that Aida does indeed have a smaller range of adaptive stepsizes than AdaBelief and Adam. Furthermore, as $K$ increases from 1 to 2, the range of adaptive stepsizes of Aida becomes increasingly compact. Hence, Aida is closer to SGD with momentum than AdaBelief. As will be demonstrated in the experiments, Aida improves the generalization of Adam and AdaBelief for several classical DNN tasks.

\subsection{Convergence analysis}

In this paper, we focus on convex optimization for SET-Adam. Our analysis follows a strategy similar to that used to analyse AdaBelief in \cite{Zhuang20Adabelief}. %Note that the upper bound we obtain is essentially tighter than that in \cite{Zhuang20Adabelief} due to two minor corrections. %\footnote{1. As will be shown later, we do not replace $\beta_{1t}$ by $\beta_1$ when dealing with the quantity $\frac{1}{2\eta_t(1-\beta_{1t})}[\|\tilde{\boldsymbol{w}}_t^{1/2}(\boldsymbol{\theta}_{t-1}-\boldsymbol{\theta}^{\ast})\|_2^2 -\|\tilde{\boldsymbol{w}}_t^{1/2}(\boldsymbol{\theta}_{t}-\boldsymbol{\theta}^{\ast})\|_2^2 ]$ as is done in the derivation of (3) in the appendix of \cite{Zhuang20Adabelief}. We have added the dimensionality $d$ to the last quantity of (6) in the appendix of \cite{Zhuang20Adabelief}. } In particular, the first term in (\ref{equ:Aida_convex}) is of order $O(1/T)$ while the corresponding one in \cite{Zhuang20Adabelief} is essentially of order $1/(\sqrt{T})$.  In other words, we have improved the regret bound of \cite{Zhuang20Adabelief}.

\begin{theorem}
Suppose $\{\boldsymbol{\theta}_t\}_{t=0}^T$ and $\{\tilde{\boldsymbol{w}}_t\}_{t=0}^T$ are the iterative updates obtained by running SET-Adam\footnote{$\beta_1$ in Algorithm~1 is generalized to be $\beta_{1t}$, $t\geq 0$ to facilitate convergence analysis. AdaBelief was analyzed in a similar manner. } starting with $(\boldsymbol{m}_0,\boldsymbol{v}_0)=(\boldsymbol{0},\boldsymbol{0})$. Let  $0\leq\beta_{1t}=\beta_1\lambda^t<1, 0\leq\beta_2<1$, and  $\eta_t =\frac{\eta}{\sqrt{t}}$. Assume (1): $f(\boldsymbol{\theta})$ is a differentiable convex function with $\|\boldsymbol{g}_t\|_{\infty}\leq G_{\infty}\sqrt{1-\beta_2}$ for all $t\in [T]$; (2): the updates $\{\boldsymbol{\theta}_t\}_{t=0}^T$ and the optimal solution $\boldsymbol{\theta}^{\ast}$ are bounded by a hyper-sphere, i.e., $\|\boldsymbol{\theta}_t\|_2\leq D$ and $\|\boldsymbol{\theta}^{\ast}\|_2\leq D$; (3): $0<\sqrt{c}\leq \tilde{\boldsymbol{w}}_{t}[i] \leq \tilde{\boldsymbol{w}}_{t-1}[i]$ for all $i\in \{1,\ldots, d\} $ and $t\in [T]$. Denote $\bar{\boldsymbol{\theta}}_{T}=\frac{1}{T}\sum_{t=0}^{T-1} \boldsymbol{\theta}_t$ and ${\boldsymbol{g}_{1:T}^2[i]}=((\boldsymbol{g}_1[i])^2,\ldots, (\boldsymbol{g}_T[i])^2)\in \mathbb{R}^{T}$. We then have the following bound on regret: 
{\small\begin{align}
&f(\bar{\boldsymbol{\theta}}_{T})-f(\boldsymbol{\theta}^{\ast})\leq  \frac{2D^2d(G_{\infty}+\sqrt{\epsilon})}{\eta(1-\beta_1)T} +\frac{2D^2d(G_{\infty}+\sqrt{\epsilon})}{\sqrt{T}(1-\beta_1)\eta} \nonumber \\ &+\frac{(1+\beta_1)\eta\sqrt{1+\log{T}}}{2\sqrt{c}(1-\beta_1)^3 T}\sum_{i=1}^d\left\|\boldsymbol{g}_{1:T}^2[i]\right\|_2 \hspace{-0.7mm}+\hspace{-0.7mm} \frac{2D^2d\beta_1(G_{\infty}+\sqrt{\epsilon})}{(1\hspace{-0.7mm}-\hspace{-0.7mm}\beta_1)(1\hspace{-0.7mm}-\hspace{-0.7mm}\lambda)^2\eta T}. \label{equ:Aida_convex}
\end{align}} 
\label{theorem:convex}
\end{theorem}
\begin{proof}
see Appendix~\ref{appendix:convex_converge} for the proof. 
\end{proof}

\begin{remark}
The condition $0<\sqrt{c}\leq \tilde{\boldsymbol{w}}_t[i]$ in Theorem \ref{theorem:convex} follows directly from (\ref{equ:w_lower_bound}). The assumptions $\tilde{\boldsymbol{w}}_t[i]\leq \tilde{\boldsymbol{w}}_{t-1}[i]$ for all $(i,t)$ are also reasonable as $t$ increases, the gradient-magnitudes tend to approach to zero.    
\end{remark}

%The obtained regret bound is tighter than the one derived in \cite{Zhuang20Adabelief} for AdaBelief due to improved mathematical derivation. In particular, the first term in Theorem~\ref{theorem:convex} is of order $O(1/T)$ while the corresponding one in \cite{Zhuang20Adabelief} is essentially of order $O(1/\sqrt{T})$. In other words, we improved the regret bound of \cite{Zhuang20Adabelief}.
%The above theorem implies that the regret of Aida is upper bounded by $\mathcal{O}(1/\sqrt{T})$, which is consistent with the results of \cite{Zhuang20Adabelief} for AdaBelief. Similar regret bounds can also be obtained for LAdam and LAdabelief by using the assumption $0<c\leq q_{l,t-1} \leq q_{l,t}$ for all $l\in [L] $ and $t\in [T]$.  

\vspace{-0mm}
\section{Experiments}
\label{sec:exp}

We evaluated SET-Adam on three types of DNN tasks: (1) natural language processing (NLP) on training transformer and LSTM models; (2) image classification on training VGG and ResNet \cite{He15ResNet} models; (3) image generation on training WGAN-GP \cite{Gulrajani17WGANGP}.  Two open-source repositories\footnote{
``https://github.com/jadore801120/attention-is-all-you-need-pytorch" is adopted for the task of training a transformer, which produces reasonable validation performance using Adam.
``https://github.com/juntang-zhuang/Adabelief-Optimizer" is adopted for all the remaining tasks. The second open source is the original implementation of AdaBelief \cite{Zhuang20Adabelief}. } were used for the above DNN training tasks. To demonstrate the effectiveness of the proposed method, eight adaptive optimizers from the literature were tested and compared, namely Yogi  \cite{Zaheer18Yogi}, RAdam \cite{Liu19RAdam}, MSVAG  \cite{Balles17MSVAG}, Fromage \cite{Bernstein20Fromage}, Adam \cite{Kingma17}, AdaBound \cite{Luo19AdaBound},  AdamW \cite{Loshchilov19AdamW}, and AdaBelief \cite{Zhuang20Adabelief}. In addition, SGD with momentum was evaluated as a baseline for performance comparison. In all experiments, the additional parameter $\tau$ in SET-Adam was set to $\tau=0.5$, and not tunned for each DNN task for simplicity.  

It is found that SET-Adam outperforms the eight adaptive optimizers for training transformer, LSTM, VGG11, and ResNet34 models while it matches the best performance of the eight methods for training WGAN-GP. The non-adaptive method SGD with momentum produce good performance only when training VGG11 and ResNet34. Lastly, experiments on training ResNet18 on the large ImageNet dataset show that SET-Adam outperforms Adam and AdaBelief. 

%either better or competitive in comparison to a number of popular adaptive optimization methods.  

The time complexity of SET-Adam was evaluated for training VGG11 and ResNet34 on a 2080 Ti GPU. In brief, SET-Adam consumed $12\%-20\%$ more time per epoch compared to Adam.   

\subsection{On training a transformer}
 In this task, we consider the training of a transformer for WMT16: multimodal translation by using the first open-source as indicated in the footnote.  In the training process, we retained almost all of the default hyper-parameters provided in the open-source except for the batch size. Due to limited GPU memory, we changed the batch size from 256 to 200.  
The parameters of SET-Adam were set to  $(\eta_0,\beta_1, \beta_2, \epsilon, \tau)=(0.001, 0.9, 0.98, 1e-15, 0.5)$. The parameter-setups for other optimizers can be found in Table~\ref{tab:setup_transformer} of Appendix~\ref{appendix:parameterSetups}, where certain hyper-parameters for each optimizer were searched over some discrete sets to optimize the validation performance. For example, the parameter $\epsilon$ of Adam was searched over the set $\{1e-6,1e-7,\ldots, 1e-12\}$ while the remaining parameters were set to $(\eta_0,\beta_1,\beta_2)=(0.001, 0.9, 0.98)$ as in SET-Adam.  Once the optimal parameter-configuration for each optimizer was obtained by searching, three experimental repetitions were then performed to alleviate the effect of the randomness. 

It is clear from Table~\ref{tab:transformer_val_acc} that SET-Adam significantly outperforms all other methods. We emphasize that the maximum number of epochs was set to 400 for each optimizer by following the default setup in the open source, and no epoch cutoff is performed in favor of SET-Adam. Fig.~\ref{fig:trans_comapre} demonstrates that SET-Adam not only converges faster in the training process than Adam but also produces much better validation accuracy. On the other hand, the non-adaptive method SGD with momentum produces a performance that is inferior to all adaptive methods except Fromage and MSVAG.  

\begin{table}[t!]
\caption{Performance comparison for
training the transformer.}
\label{tab:transformer_val_acc}
\centering
  \begin{tabular}{|c|c||c|c|}
  \hline
  \hspace{-3.5mm}  {\footnotesize $\begin{array}{c}\textrm{SGD}\\ \textrm{(non-adaptive)}\end{array}$} \hspace{-3.5mm} & \hspace{-2mm} \footnotesize{55.58$\pm$0.34}
    \hspace{-2mm} & \hspace{-2mm} 
  {\footnotesize AdaBound}  \hspace{-2mm} & \hspace{-2mm} \footnotesize{55.90$\pm$0.21}    
   \hspace{-2mm} \\ 
  \hline
   \hspace{-3.5mm} \footnotesize{Yogi} \hspace{-3.5mm} & \hspace{-2mm} \footnotesize{60.47$\pm$0.61} \hspace{-2mm} & \hspace{-2mm}
  \footnotesize{RAdam} \hspace{-2mm} & \hspace{-2mm} \footnotesize{64.47$\pm$0.19} \hspace{-2mm}   
  \\ \hline
  \hspace{-3.5mm} \footnotesize{MSVAG} \hspace{-3.5mm} & \hspace{-2mm} \footnotesize{53.79$\pm$0.13} 
   \hspace{-2mm} & \hspace{-2mm} \footnotesize{Fromage} \hspace{-2mm} & \hspace{-2mm} \footnotesize{35.57$\pm$0.19} \hspace{-2mm}
   \\ \hline 
   \hspace{-3mm}  \footnotesize{AdamW} \hspace{-2mm} & \hspace{-3.5mm} \footnotesize{64.49$\pm$0.24}
     \hspace{-3.5mm} & \hspace{-2mm}  \footnotesize{AdaBelief} \hspace{-2mm} & \hspace{-2mm} \footnotesize{66.90$\pm$0.77} \hspace{-2mm}
   \\    \hline 
   \hspace{-3.5mm} \footnotesize{Adam} \hspace{-3.5mm} & \hspace{-2mm} \footnotesize{64.71$\pm$0.57} 
   \hspace{-2mm} & \hspace{-2mm}  
  \footnotesize{SET-Adam (\textbf{Our})} \hspace{-2mm} & \hspace{-2mm} \footnotesize{\textbf{69.19}$\pm$0.09} \hspace{-2mm}
    \\ \hline

  \end{tabular}
\end{table}

\begin{figure}[t!]
\centering
\includegraphics[width=70mm]{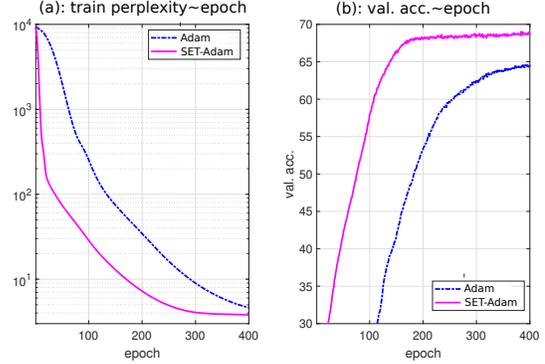}
\vspace*{-0.0cm}
\caption{\small{Performance visualisation of Adam and SET-Adam for the training of the transformer. } } 
\label{fig:trans_comapre}
\vspace*{-0.0cm}
\end{figure}

\begin{table*}[t!]
\caption{\small  Validation perplexity on Penn Treebank for 1, 2, 3-layer LSTM. \textbf{lower} is better.  \vspace{-0mm} } 
\label{tab:LSTM}
\centering
\begin{tabular}{|c|c|c|c|c|c|}
\hline
& \hspace{0mm} {\footnotesize $\begin{array}{c}\textrm{SGD}\\ \textrm{(non-adaptive)}\end{array}$} \hspace{0mm}
& \hspace{0mm}{\footnotesize   AdaBelief}   \hspace{0mm} & {\footnotesize AdamW }
 & {\footnotesize Yogi}  & \footnotesize{AdaBound} \\
\hline 
\footnotesize{1 layer} & 
\footnotesize{ 85.52 } &
\footnotesize{ 84.21} &
\footnotesize{88.36} & \footnotesize{86.78} & \footnotesize{84.52}
\\ \hline 
\footnotesize{2 layer} &
\footnotesize{67.44} &
\footnotesize{66.29} &
\footnotesize{73.18}   &  \footnotesize{71.56}  & \footnotesize{67.01}
\\ \hline 
\footnotesize{3 layer} &
\footnotesize{63.68} &
\footnotesize{61.23} &
\footnotesize{70.08}   &  \footnotesize{67.83}  & \footnotesize{63.16} \\ 
\hline
\hline
%\hline % inserts single-line
%\cline{1-6}
& \footnotesize{SET-Adam}  & \footnotesize{Adam} & {\footnotesize RAdam}
 & {\footnotesize MSVAG} & {Fromage}
 \\
\hline
\footnotesize{1 layer}& \footnotesize{\textbf{78.54}} 
& \footnotesize{84.28}
 & \footnotesize{88.76}
 & \footnotesize{84.75}
 & \footnotesize{85.20}
 \\
\hline
\footnotesize{2 layer}& \footnotesize{\textbf{64.80}}
& \footnotesize{66.86}
 & \footnotesize{74.12}
 & \footnotesize{ 68.91}
 & \footnotesize{ 72.22} 
 \\
\hline
\footnotesize{3 layer}& \footnotesize{\textbf{60.86}}
& \footnotesize{64.28}
 & \footnotesize{70.41}
 & \footnotesize{65.04}
 & \footnotesize{67.37}
 \\
\hline
\end{tabular}
\vspace{-1mm}
\end{table*}

\begin{figure*}[t!]
\centering
\includegraphics[width=140mm]{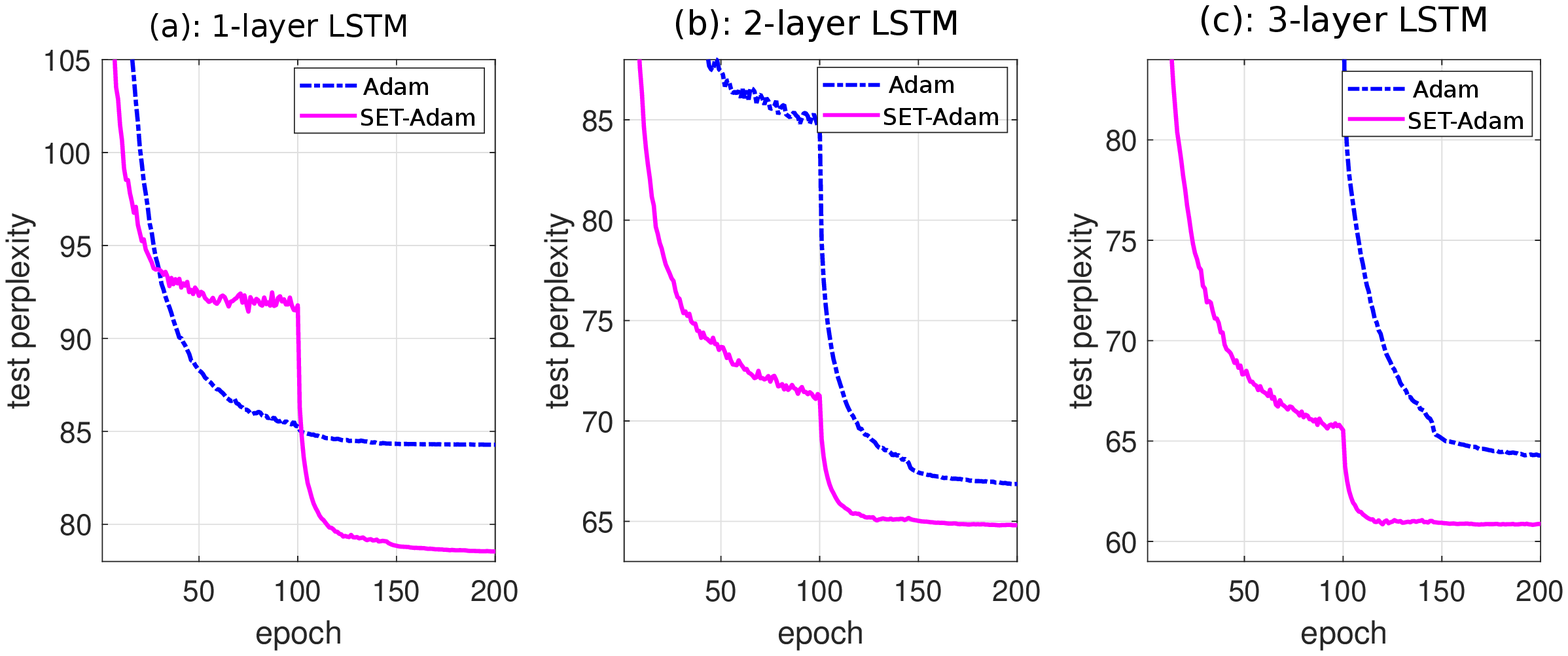}
\vspace*{-0.0cm}
\caption{\footnotesize{Performance visualisation of SET-Adam and Adam in Table~\ref{tab:LSTM}. } }
\label{fig:LSTM_comapre}
\vspace*{-0.3cm}
\end{figure*}

\begin{table*}[t]
\caption{\footnotesize Validation accuracies (in percentage) and \textcolor{blue}{time complexity in seconds per epoch (referred to as \emph{t. c.})} of ten methods for training VGG11 and ResNet34 over CIFAR10 and CIFAR100. The best result is highlighted in each column. \vspace{0mm}
} 
\label{tab:val_acc_imgclass}
\centering
%\begin{tabular}{ |*{8}{c|} } 
\begin{tabular}{|c|c|c|c|c||c|c|c|c|}
\cline{2-9} 
%\hline
%{\scriptsize \multicolumn{9}{c}{deterministic} }
\multicolumn{1}{c|}{} & \multicolumn{4}{|c||}{{\footnotesize CIFAR10 }} & \multicolumn{4}{|c|}{{\footnotesize CIFAR100 }}      \\
\cline{2-9} %\cline{1-3}
 \multicolumn{1}{c|}{} & \multicolumn{2}{|c|}{\footnotesize{VGG11}} & \multicolumn{2}{|c||}{\footnotesize{ResNet34}} & \multicolumn{2}{|c|}{\footnotesize{VGG11}} & \multicolumn{2}{|c|}{\footnotesize{ResNet34}} %& \scriptsize{X}& \scriptsize{X} 
 \\ % Entering row contents Midnight&7&-3& 5&3&-1&-3&5\\
 \hline %\cline{1-3} %\cline{7-8}
 \multicolumn{1}{|c|}{{\footnotesize optimizers}} & \footnotesize{val. acc} & \footnotesize{t.~c.} & \footnotesize{val. acc} & \footnotesize{t.~c.} &
 \footnotesize{val. acc} & \footnotesize{t.~c.} &
 \footnotesize{val. acc} & \footnotesize{t.~c.} 
 \\
 \hline
{\footnotesize $\begin{array}{c}\textrm{SGD}\\ \textrm{(non-adaptive)}\end{array}$} \hspace{-2mm} & \hspace{-2mm} \footnotesize{91.36$\pm$0.07} & \footnotesize{\textbf{5.83}} \hspace{-2mm} & \hspace{-2mm} {\footnotesize{\textbf{95.48}$\pm$0.11}} & \footnotesize{\textbf{30.45}} \hspace{-2mm} & \footnotesize{ 67.02$\pm$0.25} & \footnotesize{\textbf{5.85}} &  {\footnotesize{\textbf{78.10}$\pm$0.18}} & \footnotesize{\textbf{30.92}} %& \hspace{-2mm} \scriptsize{X} \hspace{-2mm} & \hspace{-2mm} \scriptsize{X}  \hspace{-2mm} 
 \\ 
\hline
\hspace{-2mm} {\footnotesize Yogi} \hspace{-2mm} & \hspace{-2mm} \footnotesize{90.74$\pm$0.16} & \footnotesize{6.49}  \hspace{-2mm} & \hspace{-2mm} \footnotesize{94.98$\pm$0.26}  & \footnotesize{31.74}  \hspace{-2mm} &  \footnotesize{65.57$\pm$0.17} & \footnotesize{6.42} &  \footnotesize{77.17}$\pm$0.12 & \footnotesize{32.20} %& \hspace{-2mm} \scriptsize{X} \hspace{-2mm} & \hspace{-2mm} \scriptsize{X}  \hspace{-2mm} 
\\ % Entering row contents Midnight&7&-3& 5&3&-1&-3&5\\
\hline %\cline{1-3} % \cline{7-8}
{\footnotesize RAdam} \hspace{-2mm} & \hspace{-2mm} \footnotesize{89.58$\pm$}0.10 & \footnotesize{6.28} \hspace{-2mm} & \hspace{-2mm} \footnotesize{94.64$\pm$0.18} & \footnotesize{31.21} \hspace{-2mm} & \footnotesize{63.62$\pm$0.20} & \footnotesize{6.29} & \footnotesize{74.87}$\pm$0.13 & \footnotesize{31.58} %& \hspace{-2mm} \scriptsize{X} \hspace{-2mm} & \hspace{-2mm} \scriptsize{X}  \hspace{-2mm}  
\\
\hline %\cline{1-3} % \cline{7-8}
{\footnotesize MSVAG} \hspace{-2mm} & \hspace{-2mm} \footnotesize{90.04$\pm$0.22} & \footnotesize{7.08} \hspace{-2mm} & \hspace{-2mm} \footnotesize{94.65$\pm$0.08} & \footnotesize{33.78} \hspace{-2mm} & \footnotesize{62.67$\pm$0.33} & \footnotesize{7.19} & \footnotesize{75.57}$\pm$0.14
& \footnotesize{33.80}
%& \hspace{-2mm} \scriptsize{X} \hspace{-2mm} & \hspace{-2mm} \scriptsize{X}  \hspace{-2mm}  
\\
 \hline %\cline{1-3} %\cline{7-8}
{\footnotesize Fromage} \hspace{-2mm} & \hspace{-2mm} \footnotesize{89.72$\pm$0.25} & \footnotesize{6.66} \hspace{-2mm} & \hspace{-2mm}  \footnotesize{94.64$\pm$0.07} & \footnotesize{35.19} \hspace{-2mm} & \footnotesize{62.93$\pm$0.53} & \footnotesize{6.56} & \footnotesize{74.84}$\pm$0.27 & \footnotesize{35.50} %& \hspace{-2mm} \scriptsize{} \hspace{-2mm} & \hspace{-2mm} \scriptsize{X}  \hspace{-2mm}  
\\
 \hline %\cline{1-3} %\cline{7-8}
{\footnotesize AdamW} \hspace{-2mm} & \hspace{-2mm}   \footnotesize{89.46$\pm$0.08} & \footnotesize{6.25} \hspace{-2mm} & \hspace{-2mm} \footnotesize{94.48$\pm$0.18} & \footnotesize{31.71} \hspace{-2mm} & \footnotesize{62.50$\pm$0.23} & \footnotesize{6.31} & \footnotesize{74.29$\pm$0.20} & \footnotesize{31.80} %& \hspace{-2mm} \scriptsize{X} \hspace{-2mm} & \hspace{-2mm} \scriptsize{X}  \hspace{-2mm}  
\\
 \hline %\cline{1-3} % \cline{7-8}
 {\footnotesize AdaBound}  \hspace{-2mm} & \hspace{-2mm} {\footnotesize{90.48$\pm$0.12}} & \footnotesize{6.71}  \hspace{-2mm} & \hspace{-2mm}  \footnotesize{94.73$\pm$0.16} & \footnotesize{33.75} \hspace{-2mm} & {\footnotesize{64.80$\pm$0.42}} %epsilon=1e-9
 & \footnotesize{6.73} & \footnotesize{76.15}$\pm$0.10 & \footnotesize{33.78}  %& \hspace{-2mm} \scriptsize{92.29} \hspace{-2mm} & \hspace{-2mm} \scriptsize{92.53}  \hspace{-2mm}  
\\
\hline
 {\footnotesize AdaBelief}  \hspace{-2mm} & \hspace{-2mm} {\footnotesize{91.55$\pm$0.13}} & \footnotesize{6.47}  \hspace{-2mm} & \hspace{-2mm}  \footnotesize{95.15$\pm$0.11} & \footnotesize{31.66} \hspace{-2mm} & {\footnotesize{68.05$\pm$0.31}} %epsilon=1e-9
 & \footnotesize{6.49} & \footnotesize{77.32}$\pm$0.37 & \footnotesize{31.74}  %& \hspace{-2mm} \scriptsize{92.29} \hspace{-2mm} & \hspace{-2mm} \scriptsize{92.53}  \hspace{-2mm}  
\\
\hline
\hline
{\footnotesize Adam} \hspace{-2mm} & \hspace{-2mm} \footnotesize{91.20$\pm$0.21} & \footnotesize{6.15} \hspace{-2mm} & \hspace{-2mm} \footnotesize{95.09$\pm$0.18} & \footnotesize{31.28} \hspace{-2mm} & \footnotesize{67.88$\pm$0.13} & \footnotesize{6.20} & \footnotesize{77.31}$\pm$0.14 & \footnotesize{31.47} %& \hspace{-2mm} \scriptsize{X} \hspace{-2mm} &d  \hspace{-2mm} \scriptsize{X}  \hspace{-2mm}  
\\ 
\hline
 {\footnotesize SET-Adam (\textbf{our})} \hspace{-2mm} & \hspace{-2mm} {\footnotesize{\textbf{91.89}$\pm$0.17}} & \footnotesize{7.40} \hspace{-2mm} & \hspace{-2mm} {\footnotesize{{95.47}$\pm$0.06}} & \footnotesize{36.25}  \hspace{-2mm}  & {\footnotesize{\textbf{69.93}$\pm$0.20}} & \footnotesize{7.21} & {\footnotesize{{77.75}}$\pm$0.45} & \footnotesize{35.10}
 \\ \hline
\end{tabular}
\vspace{-0.5mm}
\end{table*}

\vspace{-3mm}
\begin{table*}[h!]
\caption{\small Best FID obtained for each optimizer (\textbf{lower}
is better)   \vspace{0mm} } 
\label{tab:GAN_FID}
\centering
\begin{tabular}{|c|c|c|c|c|c|}
\hline
& \hspace{0mm} {{\footnotesize Adam }}\hspace{0mm}
&  {\footnotesize AdaBelief}
& {\footnotesize AdamW} &  {\footnotesize RAdam}
& {\footnotesize AdaBound}
\\ \hline
\footnotesize{best FIDs} &
\footnotesize{66.71} & 
  \footnotesize{\textbf{56.73}} & 63.76 & \footnotesize{69.14} & \footnotesize{61.65}   \\
\hline 
\hline
&  {\footnotesize SET-Adam} & {\footnotesize MSVAG} & {\footnotesize SGD} & {\footnotesize Yogi} & {\footnotesize Fromage}  \\
\hline
\footnotesize{best FIDs} &  \footnotesize{57.42} &  \footnotesize{69.47} & \footnotesize{90.61} & \footnotesize{68.34} &  \footnotesize{78.47}
\\
\hline
\end{tabular}
\end{table*}

\subsection{On training LSTMs}
In this experiment, we consider training LSTMs with a different number of layers over the Penn TreeBank dataset \cite{Marcus93PennTree}. The detailed experimental setup such as dropout rate and gradient-clipping magnitude can be found in the first open-source repository provided in the footnote. %The parameters of SET-Adam were set to $(\eta_0, \beta_1,\beta_2,\epsilon)=(0.001,0.9,0.999,1e-16)$.
Similar to the task of training the transformer, the  optimizers have both fixed and free parameters of which the free parameters remain to be searched over some discrete sets. See Table~\ref{tab:setup_LSTM} in Appendix~\ref{appendix:parameterSetups} for a summary of the fixed and free parameters for each optimizer. An example is Adam for which $\eta_0\in \{0.01, 0.001\}$ and $\epsilon\in \{1e-6, 1e-8, 1e-10, 1e-12\}$ were tested to find the optimal configuration that produces the best validation performance.   
  
Table.~\ref{tab:LSTM} summarises the obtained validation perplexities of the ten methods for training 1, 2, and 3-layer LSTMs. It was found that for each optimizer, independent experimental repetitions lead to almost the same validation perplexity value. Therefore, we only keep the average of the validation perplexity values from three independent experimental repetitions for each experimental setup in the table and ignore the standard deviations.

It is clear from Table.~\ref{tab:LSTM} that SET-Adam again outperforms all other methods in all three scenarios, which may be due to the contribution of a compact range of adaptive stepsizes in SET-Adam. Fig.~\ref{fig:LSTM_comapre} further visualised the validation performance of SET-Adam compared to Adam. The performance gain of SET-Adam is considerable in all three scenarios.   %AdaBelief performs slightly better than LAdam and LAdaBelief. This might be because, on Pytorch platform \cite{NEURIPS2019_9015}, the different small weight matrices of a single LSTM cell are put into a big matrix, of which all the parameters are treated to be from a single neural layer. It may occur that those small weight matrices have different gradient statistics as each of them process different information within the LSTM cell.    

\vspace{-2mm}
\subsection{On training VGG11 and ResNet34 over CIFAR10 and CIFAR100}\vspace{-1mm}
In this task, the ten optimizers were evaluated by following  a similar experimental setup as in \cite{Zhuang20Adabelief}. The batch size and epoch were set to 128 and 200, respectively. The common stepsize $\eta_t$ is reduced by multiplying by 0.1 at 100 and 160 epoch.  The detailed parameter-setups for the optimizers can be found in Table~\ref{tab:setup_VGGResNet} in Appendix~\ref{appendix:parameterSetups}.  Three experimental repetitions were conducted for each optimizer to alleviate the effect of randomness.     

Both the validation performance and the algorithmic complexity are summarised in Table~\ref{tab:val_acc_imgclass}. It is clear that SET-Adam consistently outperforms the eight \textcolor{blue}{adaptive} optimizers in terms of validation accuracies. This demonstrates that the compact range of adaptive stepsizes in SET-Adam does indeed improve the generalization performance. The non-adaptive optimizer SGD with momentum demonstrates good performance, confirming the findings in \cite{He15ResNet, Wilson17AdamNegative} that this optimizer dominates over adaptive optimizers in image classification tasks. 

We can also conclude from the table that SGD with momentum is the most computationally efficient method. On the other hand, due to the three operations, SET-Adam consumed an additional $12\%-20\%$ time per epoch compared to Adam.

\vspace{-2mm}
\subsection{On training WGAN-GP over CIFAR10}
\vspace{-1mm}
This task focuses on training WGAN-GP. The parameters of SET-Adam were set to $(\eta_t, \beta_1,\beta_2, \epsilon, \tau)=(0.0002, 0.5, 0.999, 1e-11, 0.5)$.
The parameter-setups of other optimizers can be found in Table~\ref{tab:setup_WGANGP} of Appendix~\ref{appendix:parameterSetups}. As an example, the parameter $\epsilon$ of Adam is searched over the discrete set $\{1e-4, 1e-6,\ldots, 1e-14\}$. 
For each parameter-configuration of an optimizer, three experimental repetitions were performed due to the relatively unstable Frechet inception distance (FID) scores in training WGAN-GP. 

Table~\ref{tab:GAN_FID} shows the best FID for each method. Considering Adam for example, it has six parameter-configuration due to six $\epsilon$ values being tested. As a result, the best FID for Adam is obtained over 18 values, accounting for three experimental repetitions for each of six $\epsilon$ values. It can be seen from the table that SET-Adam provides comparable performance to AdaBelief, while the other methods including Adam perform significantly worse.

\vspace{-1mm}
\subsection{On training ResNet18 over ImageNet} 
In the last experiment, we investigated the performance gain of SET-Adam compared to Adam and AdaBelief for training ResNet18 on the large ImageNet dataset. The maximum epoch and minibatch size were set to 90 and 256, respectively. The common stepsize $\eta_t$ is dropped by a factor of 0.1 at 70 and 80 epochs.  The parameter setup for the three optimizers can be found in Table~\ref{tab:setup_imagenet} of Appendix~\ref{appendix:parameterSetups}. %Similarly, three experimental repetitions were conducted for each optimizer to mitigate the effect of randomness. 

%We note that due to prohibitive computational effort, other optimizers were not tested in the experiment. In \cite{Zhuang20Adabelief}, AdaBelief is found to outperform the other seven optimizers when training ResNet18 over ImageNet (see Table~2 of \cite{Zhuang20Adabelief} where no std is reported).  

\begin{figure}[t!]
\centering
\includegraphics[width=80mm]{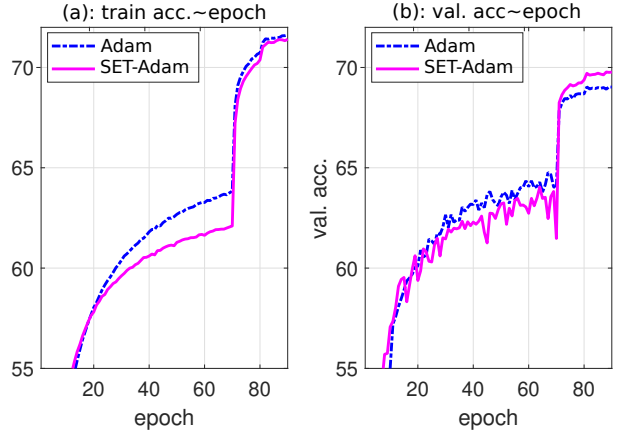}
\vspace*{-0.0cm}
\caption{\small{Performance visualisation of Adam and SET-Adam for the training of ResNet18 over ImageNet. } } 
\label{fig:imagenet_comapre}
\vspace*{-0.5cm}
\end{figure}

\begin{table}[t!]
\caption{\small Validation accuracies (in percentage) of the optimizers for training ResNet18 over ImageNet.    \vspace{-0mm} } 
\label{tab:imagenet}
\centering
\begin{tabular}{|c|c|c|c|}
\hline
 {\footnotesize optimizers} &
 \hspace{0mm} {{\footnotesize AdaBelief }}\hspace{0mm} 
 &
 \hspace{0mm} {{\footnotesize Adam }}\hspace{0mm}
&   \hspace{0mm} {{\footnotesize SET-Adam }}\hspace{0mm}  
\\ \hline
\footnotesize{val. acc.}  &  \footnotesize{69.65} & \footnotesize{69.03}  & \footnotesize{\textbf{69.77}}
\\ \hline
\end{tabular}
\vspace*{-0.5cm}
\end{table}

It is clear from Table~\ref{tab:imagenet} that for the large ImageNet dataset, SET-Adam again performs better than Adam and AdaBelief, indicating that the performance gain of SET-Adam is robust against different sizes of datasets.

Fig.~\ref{fig:imagenet_comapre} visualizes the training and validation curves of Adam and SET-Adam. Interestingly, it is seen that Adam exhibits better training accuracy and worse validation accuracy than SET-Adam. The above results are consistent with the findings in \cite{Wilson17AdamNegative}. In particular,  in \cite{Wilson17AdamNegative}, the authors empirically found that Adam performs better than SGD with momentum in the training process while producing worse validation performance in different DNN tasks. The validation performance gain of SET-Adam in Fig.~\ref{fig:imagenet_comapre} can be explained by the property that SET-Adam is designed to be closer to SGD with momentum.       

\vspace{-0mm}
\section{Conclusions}
\vspace{-0mm}

%In this paper, we have made three contributions. Firstly, we study the particular placement of $\epsilon$ in the update expressions of AdaBelief. We show that the inclusion of $\epsilon$ in the computation of the 2nd momentum $\boldsymbol{s}_t$ (see Eq.~(\ref{equ:AdaBelief})) in AdaBelief helps to reduce the range of adaptive stepsizes, making it closer to SGD with momentum. The above technique can also be applied to other Adam-like methods to reduce their range of adaptive stepsizes (e.g., see Adam+ in Appendix~\ref{appendix:adamplus}).  

In this paper, we have shown that the range of the adaptive stepsizes of DNN optimizers has a significant impact on performance. The proposed SET-Adam optimizer suppresses the range of the adaptive stepsizes of Adam making it closer to SGD with momentum. Our experimental results indicate that SET-Adam will be able to produce better performance across a wide range of DNN-based applications.

In the design of the SET-Adam optimizer, we have proposed to perform a sequence of three operations on the second momentum $\boldsymbol{v}_t$ at iteration $t$ for the purpose of reducing the range of adaptive stepsizes in Adam while avoiding extreme small ones. We realize the down-scaling operation in SET-Adam by exploiting the angles between the layerwise subvectors $\{\boldsymbol{v}_{l,t}\}_{l=1}^L$ of $\boldsymbol{v}_t$ and the corresponding all-one vectors $\{\boldsymbol{1}_l\}_{l=1}^L$.

Our empirical study shows that SET-Adam outperforms eight \textcolor{blue}{adaptive} optimizers including Adam and AdaBelief for training transformer, LSTM, VGG11, and ResNet34 models while at the same time it matches the best performance of the eight adaptive methods for training WGAN-GP models. In addition, experiments on training ResNet18 over the large ImageNet dataset show that SET-Adam performs better than Adam and AdaBelief.  On the other hand, it was found that SGD with momentum only produces good performance in image classification tasks. This suggests that the \emph{adaptivity} of SET-Adam is important, allowing the method to effectively train different types of DNN models.

\appendices

\onecolumn

%\appendices
%\clearpage
%\newpage 
%\onecolumn

\section{Verification that AdaBelief also utilizes the $\epsilon$-embedding operation}
\label{appendix:AdaBelief_embedding}
 The update expressions of AdaBelief are given by \cite{Zhuang20Adabelief}
\begin{align}
\hspace{-3mm}[\textbf{AdaBelief}]\hspace{0mm}&\left\{\begin{array}{l}
\hspace{-2mm}\boldsymbol{m}_t = \beta_1\boldsymbol{m}_{t-1}\hspace{-0.5mm} +\hspace{-0.5mm} (1\hspace{-0.5mm}-\hspace{-0.5mm}\beta_1)\boldsymbol{g}_t \\
\hspace{-2mm}\boldsymbol{s}_t=\beta_2 \boldsymbol{s}_{t-1}\hspace{-0.6mm}+\hspace{-0.6mm}(1\hspace{-0.6mm}-\hspace{-0.6mm}\beta_2)(\boldsymbol{m}_t\hspace{-0.6mm}-\hspace{-0.6mm}\boldsymbol{g}_t)^2+\textcolor{blue}{\epsilon}  \\
\hspace{-2mm}\boldsymbol{\theta}_t = \boldsymbol{\theta}_{t-1}  -\eta_t\frac{1}{1-\beta_1^t}\frac{\boldsymbol{m}_t}{\sqrt{\boldsymbol{s}_t/(1-\beta_2^t)}+\textcolor{blue}{\epsilon}}
\end{array}\right.\hspace{-2mm},
\label{equ:AdaBelief}
\end{align}
where the parameter $\epsilon$ is involved in the computation of both $\boldsymbol{s}_t$ and $\boldsymbol{\theta}_t$ in AdaBelief, which is different from that of Adam. \textcolor{blue}{The second $\epsilon$ for computing $\boldsymbol{\theta}_t$ can be ignored in (\ref{equ:AdaBelief}) since the first $\epsilon$ dominates the second one.}

By simple algebra, one can show that (\ref{equ:AdaBelief}) can be reformulated to be
\begin{align}
\hspace{-3mm}[\textbf{AdaBelief}]\hspace{0mm}&\left\{\begin{array}{l}
\hspace{-2mm}\boldsymbol{m}_t = \beta_1\boldsymbol{m}_{t-1}\hspace{-0.5mm} +\hspace{-0.5mm} (1\hspace{-0.5mm}-\hspace{-0.5mm}\beta_1)\boldsymbol{g}_t \\
\hspace{-2mm}\boldsymbol{s}_t=\beta_2 \boldsymbol{s}_{t-1}\hspace{-0.6mm}+\hspace{-0.6mm}(1\hspace{-0.6mm}-\hspace{-0.6mm}\beta_2)(\boldsymbol{m}_t\hspace{-0.6mm}-\hspace{-0.6mm}\boldsymbol{g}_t)^2  \\
\hspace{-2mm}\boldsymbol{\theta}_t = \boldsymbol{\theta}_{t-1}  -\eta_t\frac{1}{1-\beta_1^t}\frac{\boldsymbol{m}_t}{\sqrt{\boldsymbol{s}_t/(1-\beta_2^t)+\textcolor{blue}{\epsilon/(1-\beta_2)}}}
\end{array}\right.\hspace{-2mm},
\label{equ:AdaBelief_2nd}
\end{align}
where the second $\epsilon$ is ignored. The above expression shows that AdaBelief indeed utilizes the  $\epsilon$-embedding operation as we do.

%It remains unclear what the impact of $\epsilon$. 

\vspace{-2mm}
\section{Update procedure of Adam$^\star$ }
\label{appendix:adamplus}
The parameter $\epsilon$ is put inside of the sqrt $\sqrt{\cdot}$ operation to verify if Adam$^{\star}$ has a small range of adaptive stepsizes than Adam. 
\label{appendix:adamplus}
\begin{algorithm}[h!]
   \caption{\small Adam$^\star$}
   \label{alg:Adamplus}
\begin{algorithmic}[1]
   \STATE {\small {\bfseries Input:} $\beta_1$, $\beta_2$,  $\eta_t$, $\epsilon > 0$ }
   \STATE {\small {\bfseries  Init.:} $\boldsymbol{\theta}_0\hspace{-0.5mm}\in\hspace{-0.5mm} \mathbb{R}^d$,  $\boldsymbol{m}_0 \hspace{-0.5mm}=\hspace{-0.5mm} 0$, $\boldsymbol{v}_{0}=0 \in \mathbb{R}^d$ }
   \FOR{\small $t=1, 2, \ldots, T$}
   \STATE \hspace{-0mm}{\small  $\boldsymbol{g}_t \leftarrow \nabla f({\boldsymbol{\theta}}_{t-1}) $  }
   \STATE \hspace{-0mm}{\small $\boldsymbol{m}_{t} \leftarrow \beta_1 \boldsymbol{m}_{t-1}  + (1-\beta_1) \boldsymbol{g}_t$ }
    \STATE \hspace{-0mm}{\small ${v}_{t}  \leftarrow \beta_2 v_{t-1} + (1-\beta_2) \boldsymbol{g}_{t}^2$}
  % \STATE {\small $  {x}_t \leftarrow {x}_{t-1} \hspace{-0.6mm} - \hspace{-0.6mm}  \left\{\hspace{-1.5mm} \begin{array}{l}  \eta \frac{ 1}{(1-\beta_1^t)^q}  \frac{|{m}_t|^{q-1}\odot m_t}{(r_t/(1-\beta_2^t))^{q/p}+ \epsilon}    \\
  %\eta \frac{ (1-\beta_2^t))^{q/p} }{(1-\beta_1^t)^q}  \frac{|{m}_t|^{q-1}\odot m_t}{(r_t)^{q/p}+ \epsilon} 
    %\end{array} \right. $ }
   \STATE \hspace{-0mm}{\small $  \tilde{\boldsymbol{m}}_{t} \hspace{-0.6mm}\leftarrow \frac{\boldsymbol{m}_{t}}{1-\beta_1^{t}}\quad   \tilde{v}_{t} \leftarrow \frac{v_{t}}{1-\beta_2^{t}} $ }     
  \STATE {\small  $ \boldsymbol{\theta}_{t} \hspace{-0.6mm}\leftarrow \boldsymbol{\theta}_{t-1} -\frac{\eta_t }{\sqrt{\tilde{v}_{t}\textcolor{blue}{+\epsilon}}} \tilde{\boldsymbol{m}}_{t} $}
   \ENDFOR 
   \STATE {\bfseries Output:} {\small $\boldsymbol{\theta}_{T}$  }
   \\ 
  %\hrulefill \\
  %  \vspace{1mm}\hrule  width0.45\textwidth \vspace{1mm}
 %\hspace{-5mm} *\hspace{0.1mm} {\small {Typical setups}: $\beta_1=0.9$, $\beta_2=0.999$, $\epsilon=1e-8$, $\eta_0=0.001$.  }   
\end{algorithmic}
\end{algorithm}

\vspace{-2mm}
\section{Parameter Setups for Optimization Methods in  Fig.~\ref{fig:SETAdam_mean_compare}-\ref{fig:SETAdam_std_compare} and Fig.~\ref{fig:AdamPlus}}
\label{appendix:fig_setup}

The common stepsize $\eta_t$ is dropped by a factor 0.1 at 100 and 160 epochs.  The optimal parameter $\epsilon$ is searched over the set $\{10^{-2}, 10^{-3},\ldots, 10^{-8}\}$ for Adam. 
\begin{table}[h!]
\label{tab:setup_fig}
\centering
\begin{tabular}{|c|c|c|}
\hline
\footnotesize{optimizer} & \footnotesize{fixed parameters} & \footnotesize{searched parameters}  \\ \hline
\hspace{-4.5mm} 
\footnotesize{Adam} \hspace{-4.5mm} & \hspace{-4.5mm} 
\footnotesize{$\begin{array}{l}(\eta_0,\beta_1, \beta_2) \\
=(0.001,0.9, 0.999)\end{array}$} \hspace{-0.5mm} & \hspace{-3.5mm}\footnotesize{$\begin{array}{c} \epsilon\in\{10^{-2}, 10^{-3}, \ldots,  10^{-8}\} \end{array}$} \hspace{-4.5mm} 
\\
\hline 
\hspace{-4.5mm} 
\footnotesize{Adam$^\star$} \hspace{-4.5mm} & \hspace{-4.5mm}
\footnotesize{$\begin{array}{l}(\eta_0,\beta_1, \beta_2, \epsilon)\\
=(0.001, 0.9, 0.999, 10^{-5})\end{array}$} \hspace{-0.5mm}  & \hspace{-4.5mm} 
\\
\hline
%\hspace{-4.5mm} 
%\footnotesize{AdaBelief} \hspace{-4.5mm} & \hspace{-4.5mm} 
%\footnotesize{$\begin{array}{l}(\eta_0,\beta_1, \beta_2, \epsilon)\\
%=(0.001, 0.9, 0.999, 10^{-8})\end{array}$} \hspace{-4.5mm} & \hspace{-4.5mm}
%\\
%\hline
\hspace{-0.5mm} 
\footnotesize{$\begin{array}{c}\textrm{SET-Adam} \\ \textrm{(w. o. down-translating)} \end{array}$}  \hspace{-0.5mm} & \hspace{-0.5mm} \footnotesize{$\begin{array}{l}(\eta_0,\beta_1, \beta_2, \epsilon, \textcolor{blue}{\tau})\\
=(0.001, 0.9, 0.999, 10^{-5}, \textcolor{blue}{0.0})\end{array}$} \hspace{-0.5mm} & \hspace{-4.5mm}
\\
\hline
\hspace{-4.5mm} 
\footnotesize{$\begin{array}{c}\textrm{SET-Adam} \end{array}$}  \hspace{-4.5mm} & \hspace{-0.5mm} \footnotesize{$\begin{array}{l}(\eta_0,\beta_1, \beta_2, \epsilon, \tau)\\
=(0.001, 0.9, 0.999, 10^{-5}, 0.5)\end{array}$} \hspace{-0.5mm} & \hspace{-4.5mm}
\\
\hline
\end{tabular}
\vspace{-1mm}
\end{table}

\section{proof of Theorem~1}
\label{appendix:convex_converge}
\begin{proof}

%We now provide the derivations for the upper bound in the theorem. 

Firstly, we note that the first bias term $1-\beta_1^t$ in the update expressions of SET-Adam in Algorithm~1 can be absorbed into the common stepsize $\eta_t$. Therefore, we will ignore the first bias term in the following proof. 
%Secondly, the lower bound $0<c\leq \boldsymbol{v}_t[i]$%
Suppose $\boldsymbol{\theta}^{\ast}$ is the optimal solution for solving the convex optimization problem, i.e., $\boldsymbol{\theta}^{\ast}=\arg\min_{\boldsymbol{\theta}_t} f(\boldsymbol{\theta})$. Using the fact that $\boldsymbol{\theta}_t = \boldsymbol{\theta}_{t-1}-\eta_t \tilde{\boldsymbol{w}}_t^{-1}\boldsymbol{m}_t$, we have
\begin{align}
&\hspace{-65mm}\|\tilde{\boldsymbol{w}}_t^{1/2}(\boldsymbol{\theta}_t-\boldsymbol{\theta}^{\ast}) \|_2^2 \nonumber \\
&\hspace{-65mm}= 
\|\tilde{\boldsymbol{w}}_t^{1/2}(\boldsymbol{\theta}_{t-1}-\eta_t\tilde{\boldsymbol{w}}_{t}^{-1}\boldsymbol{m}_t-\boldsymbol{\theta}^{\ast})  \|_2^2 \nonumber \\
&\hspace{-65mm}=\|\tilde{\boldsymbol{w}}_t^{1/2}(\boldsymbol{\theta}_{t-1}-\boldsymbol{\theta}^{\ast}) \|_2^2+\eta_t^2\|\tilde{\boldsymbol{w}}_t^{-1/2}\boldsymbol{m}_t\|_2^2 \nonumber \\
& \hspace{-65mm} \hspace{5mm}- 2\eta_t\langle \beta_{1t}\boldsymbol{m}_{t-1}+(1-\beta_{1t})\boldsymbol{g}_t, \boldsymbol{\theta}_{t-1}-\boldsymbol{\theta}^{\ast} \rangle \nonumber \\
&\hspace{-65mm}=\|\tilde{\boldsymbol{w}}_t^{1/2}(\boldsymbol{\theta}_{t-1}-\boldsymbol{\theta}^{\ast}) \|_2^2+\eta_t^2\|\tilde{\boldsymbol{w}}_t^{-1/2}\boldsymbol{m}_t\|_2^2 \hspace{-0.6mm} \nonumber \\
&\hspace{-65mm}\hspace{5mm}-\hspace{-0.6mm} 2\eta_t(1\hspace{-0.6mm}-\hspace{-0.6mm}\beta_{1t})\langle \boldsymbol{g}_t, \boldsymbol{\theta}_{t-1}\hspace{-0.6mm}-\hspace{-0.6mm}\boldsymbol{\theta}^{\ast} \rangle \hspace{-0.6mm}-\hspace{-0.6mm} 2\eta_t\beta_{1t}\langle \boldsymbol{m}_{t-1},  \boldsymbol{\theta}_{t-1}\hspace{-0.6mm}-\hspace{-0.6mm}\boldsymbol{\theta}^{\ast} \rangle \nonumber  \\
&\hspace{-65mm}\hspace{-0mm}\leq \|\tilde{\boldsymbol{w}}_t^{1/2}(\boldsymbol{\theta}_{t-1}-\boldsymbol{\theta}^{\ast}) \|_2^2+\eta_t^2\|\tilde{\boldsymbol{w}}_t^{-1/2}\boldsymbol{m}_t\|_2^2 \hspace{-0.6mm} \nonumber \\
&\hspace{-65mm} \hspace{5mm}-\hspace{-0.6mm} 2\eta_t(1-\beta_{1t})\langle \boldsymbol{g}_t, \boldsymbol{\theta}_{t-1}-\boldsymbol{\theta}^{\ast} \rangle \hspace{-0.6mm} \nonumber \\
&\hspace{-65mm} \hspace{5mm}+ \eta_t^2\beta_{1t}\|\tilde{\boldsymbol{w}}_t^{-1/2}\boldsymbol{m}_{t-1}\|_2^2+\beta_{1t} \|\tilde{\boldsymbol{w}}_t^{1/2}(\boldsymbol{\theta}_{t-1}-\boldsymbol{\theta}^{\ast})\|_2^2,
\label{equ:convex1}
\end{align}
where the above inequality uses the Cauchy-Schwartz inequality $2\langle\boldsymbol{a},\boldsymbol{b}\rangle\leq \|\boldsymbol{a}\|_2^2+\|\boldsymbol{b}\|_2^2$. Note that (\ref{equ:convex1}) corresponds to (2) in the appendix of \cite{Zhuang20Adabelief} for AdaBelief.  

Summing (\ref{equ:convex1}) from $t=1$ until $t=T$,  rearranging the quantities, and exploiting the property that $\boldsymbol{g}_t=\nabla f(\boldsymbol{\theta}_{t-1})$ and $f(\cdot)$ being convex gives 
\begin{align}
&\hspace{-0mm}f(\bar{\boldsymbol{\theta}}_T) - f(\boldsymbol{\theta}^{\ast}) \nonumber \\
&\hspace{-00mm}=f\left(\frac{1}{T}\sum_{t=0}^{T-1}\boldsymbol{\theta}_t\right) - f(\boldsymbol{\theta}^{\ast}) \nonumber \\
&\hspace{-0mm}\stackrel{(a)}{\leq} \frac{1}{T}\sum_{t=1}^{T}\left(f(\boldsymbol{\theta}_{t-1}) - f(\boldsymbol{\theta}^{\ast})\right) \nonumber\\
&\hspace{-0mm}\stackrel{(b)}{\leq} \frac{1}{T}\sum_{t=1}^{T}\langle \boldsymbol{g}_t, \boldsymbol{\theta}_{t-1}-\boldsymbol{\theta}^{\ast} \rangle%\;\;\;\;\;\;\;\;\;\;\;\;\;\;\;
%\;\;\;\;\;\;\;\;\;\;\;\;\;\;\;
%\;\;\;\;\;\;\;\;\;\;\;\;\;\;\;
%\;\;\;\;\;\;\;\;\;\;\;\;\;\;\;
%\;\;\;\;\;\;\;\;\;\;\;\;\;\;\;
%\;\;\;\;\;\;\;\;\;\;\;\;\;\;\;
%\;\;\;\;\;\;\;\;\;\;\;\;\;\;\;
\nonumber \\
&\leq \frac{1}{T}\sum_{t=1}^T\Big[ \frac{1}{2\eta_t(1-\beta_{1t})} \Big(\|\tilde{\boldsymbol{w}}_t^{1/2}(\boldsymbol{\theta}_{t-1}-\boldsymbol{\theta}^{\ast}) \|_2^2  -\|\tilde{\boldsymbol{w}}_t^{1/2}(\boldsymbol{\theta}_t-\boldsymbol{\theta}^{\ast}) \|_2^2 \Big) \nonumber \\
&\hspace{14mm} +\frac{\eta_t}{2(1\hspace{-0.6mm}-\hspace{-0.6mm}\beta_{1t})}\|\tilde{\boldsymbol{w}}_t^{-1/2}\boldsymbol{m}_t\|_2^2 \hspace{-0.6mm}+\hspace{-0.6mm} \frac{\eta_t\beta_{1t}}{2(1 \hspace{-0.6mm}-\hspace{-0.6mm}\beta_{1t})}\|\tilde{\boldsymbol{w}}_t^{-1/2}\boldsymbol{m}_{t-1}\|_2^2 + \frac{\beta_{1t}}{2\eta_t(1-\beta_{1t})}\|\tilde{\boldsymbol{w}}_t^{1/2}(\boldsymbol{\theta}_{t-1}-\boldsymbol{\theta}^{\ast})\|_2^2 \Big] \nonumber \\
&\stackrel{\beta_{11}=\beta_1}{\leq}\frac{1}{2\eta(1-\beta_1)T}\|\tilde{\boldsymbol{w}}_1^{1/2}(\boldsymbol{\theta}_0-\boldsymbol{\theta}^{\ast})\|_2^2 \nonumber \\
&\hspace{5mm}+\frac{1}{T}\sum_{t=1}^{T-1}\Bigg(\frac{1}{2\eta_{t+1}(1-\beta_{1(t+1)})}\|\tilde{\boldsymbol{w}}_{t+1}^{1/2}(\boldsymbol{\theta}_t-\boldsymbol{\theta}^{\ast})\|_2^2 -\frac{1}{2\eta_t(1-\beta_{1t})}\|\tilde{\boldsymbol{w}}_t^{1/2}(\boldsymbol{\theta}_t-\boldsymbol{\theta}^{\ast})\|_2^2\Bigg) \nonumber \\
&\hspace{2mm} +\hspace{-0.7mm}\frac{1}{T}\sum_{t=1}^T\hspace{-0mm}\Big[\frac{\eta_t}{2(1\hspace{-0.7mm}-\hspace{-0.7mm}\beta_{1t})}\|\tilde{\boldsymbol{w}}_t^{-1/2}\boldsymbol{m}_t\|_2^2+\hspace{-0.6mm}\frac{\eta_t\beta_{1t}}{2(1\hspace{-0.7mm}-\hspace{-0.7mm}\beta_{1t})}\|\tilde{\boldsymbol{w}}_t^{-1/2}\boldsymbol{m}_{t-1}\|_2^2 + \frac{\beta_{1t}}{2\eta_t(1-\beta_{1t})}\|\tilde{\boldsymbol{w}}_t^{1/2}(\boldsymbol{\theta}_{t-1}-\boldsymbol{\theta}^{\ast})\|_2^2 \Big] \nonumber \\
&\hspace{3mm}\textcolor{blue}{\left(\textrm{condition: }\left\{ \begin{array}{l}0\leq \tilde{\boldsymbol{w}}_{t}[i]\leq \tilde{\boldsymbol{w}}_{t-1}[i] \textrm{ for all } i=1,\ldots, d, \\ 0\leq \eta_t\leq \eta_{t-1}, 0\leq \beta_{1(t+1)}\leq \beta_{1t}<1\end{array}\right)\right.} \nonumber \\
&\leq \frac{1}{2\eta(1-\beta_1)T}\|\tilde{\boldsymbol{w}}_1^{1/2}(\boldsymbol{\theta}_0-\boldsymbol{\theta}^{\ast})\|_2^2 \nonumber\\
&\hspace{5mm}+\frac{1}{T}\sum_{t=1}^{T-1}\Bigg(\frac{1}{2\eta_{t+1}(1-\beta_{1t})}\|\tilde{\boldsymbol{w}}_{t+1}^{1/2}(\boldsymbol{\theta}_t-\boldsymbol{\theta}^{\ast})\|_2^2\nonumber -\frac{1}{2\eta_{t}(1-\beta_{1t})}\|\tilde{\boldsymbol{w}}_{t+1}^{1/2}(\boldsymbol{\theta}_t-\boldsymbol{\theta}^{\ast})\|_2^2\Bigg) \nonumber \\
&\hspace{2mm} +\hspace{-0.7mm}\frac{1}{T}\hspace{-0.7mm}\sum_{t=1}^T\hspace{-0.7mm}\Big[\frac{\eta_t}{2(1\hspace{-0.7mm}-\hspace{-0.7mm}\beta_{1})}\|\tilde{\boldsymbol{w}}_{t+1}^{-1/2}\boldsymbol{m}_t\|_2^2\hspace{-0.7mm}+\hspace{-0.7mm}\frac{\eta_t\beta_{1}}{2(1\hspace{-0.7mm}-\hspace{-0.7mm}\beta_{1})}\|\tilde{\boldsymbol{w}}_{t}^{-1/2}\boldsymbol{m}_{t-1}\|_2^2 + \frac{\beta_{1t}}{2\eta_t(1-\beta_{1})}\|\tilde{\boldsymbol{w}}_t^{1/2}(\boldsymbol{\theta}_{t-1}-\boldsymbol{\theta}^{\ast})\|_2^2 \Big] \nonumber
\end{align}
\begin{align}
&\hspace{3mm}\textcolor{blue}{(\textrm{condition:} 0\leq \eta_t\leq \eta_{t-1}, \boldsymbol{m}_0=\boldsymbol{0})} \nonumber \\
&\leq \frac{1}{2\eta(1-\beta_1)T}\|\tilde{\boldsymbol{w}}_1^{1/2}(\boldsymbol{\theta}_0-\boldsymbol{\theta}^{\ast})\|_2^2 
+\frac{1}{T}\sum_{t=1}^{T-1} \frac{1/\eta_{t+1}-1/\eta_t}{2(1-\beta_1)} \|\tilde{\boldsymbol{w}}_{t+1}^{1/2}(\boldsymbol{\theta}_t-\boldsymbol{\theta}^{\ast})\|_2^2 \nonumber \\
&\hspace{3mm}+\frac{1}{T}\sum_{t=1}^T\frac{\eta_t(1+\beta_1)}{2(1-\beta_{1})}\|\tilde{\boldsymbol{w}}_{t+1}^{-1/2}\boldsymbol{m}_t\|_2^2  + \frac{1}{T(1-\beta_{1})}\sum_{t=1}^T \frac{\beta_{1t}}{2\eta_t}\|\tilde{\boldsymbol{w}}_t^{1/2}(\boldsymbol{\theta}_{t-1}-\boldsymbol{\theta}^{\ast})\|_2^2 \nonumber \\
&\hspace{3mm}\textcolor{blue}{(\textrm{condition: }\|\boldsymbol{\theta}^{\ast}\|_{\infty} \leq D,   \|\boldsymbol{\theta}_t\|_{\infty} \leq D)}  \nonumber \\
&\hspace{0mm}\leq \frac{2D^2}{\eta(1\hspace{-0.7mm}-\hspace{-0.7mm}\beta_1)T}\sum_{i=1}^d (\tilde{\boldsymbol{w}}_1[i]) \hspace{-0.7mm} + \frac{2D^2}{T(1-\beta_1)} \sum_{t=1}^{T-1}(1/\eta_{t+1}-1/\eta_t)\sum_{i=1}^d\tilde{\boldsymbol{w}}_t[i] \nonumber\\
&\hspace{3mm}+\hspace{-0.7mm}\frac{1}{T}\sum_{t=1}^T\frac{\eta_t(1\hspace{-0.7mm}+\hspace{-0.7mm}\beta_1)}{2(1\hspace{-0.7mm}-\hspace{-0.7mm}\beta_{1})}\|\tilde{\boldsymbol{w}}_{t+1}^{-1/2}\boldsymbol{m}_t\|_2^2  + \frac{2D^2}{T(1-\beta_{1})}\sum_{t=1}^T \frac{\beta_{1t}}{\eta_t}\sum_{i=1}^d \tilde{\boldsymbol{w}}_t[i] \nonumber \\
&\hspace{0mm}\stackrel{(c)}{\leq} \frac{2D^2d(G_{\infty}+\sqrt{\epsilon})}{\eta(1\hspace{-0.7mm}-\hspace{-0.7mm}\beta_1)T} \hspace{-0.7mm}+\frac{2D^2(G_{\infty}+\sqrt{\epsilon})d}{\sqrt{T}(1-\beta_1)\eta} +\hspace{-0.7mm}\frac{1}{T}\sum_{t=1}^T\frac{\eta_t(1\hspace{-0.7mm}+\hspace{-0.7mm}\beta_1)}{2(1\hspace{-0.7mm}-\hspace{-0.7mm}\beta_{1})}\|\tilde{\boldsymbol{w}}_{t+1}^{-1/2}\boldsymbol{m}_t\|_2^2  + \frac{2D^2\beta_1(G_{\infty}+\sqrt{\epsilon})d}{T(1-\beta_{1})\eta (1-\lambda)^2}, %\\
%&\leq \frac{2D^2d\sqrt{G_{\infty}^2+\epsilon}}{\eta(1\hspace{-0.7mm}-\hspace{-0.7mm}\beta_1)T} \hspace{-0.7mm}+\frac{2D^2\sqrt{G_{\infty}^2+\epsilon}d}{\sqrt{T}(1-\beta_1)\eta} + \frac{\eta(1+\beta_1)\sqrt{1+\log T}}{2T\sqrt{c}(1-\beta_1)^3}\|(\boldsymbol{g}_{1:T}^2[i])\|_2  + \frac{2D^2\beta_1\sqrt{G_{\infty}^2+\epsilon}d}{T(1-\beta_{1})\eta (1-\lambda)^2}
\label{equ:convex2}
\end{align}
where both step~$(a)$ and $(b)$ use the property of $f(\cdot)$ being convex, and step~$(c)$ uses the following conditions 
\begin{align}
\textcolor{blue}{\left\{\hspace{-2mm}\begin{array}{l}\|\boldsymbol{g}_t\|_{\infty}\hspace{-0.7mm}\leq\hspace{-0.7mm} G_{\infty}\sqrt{1-\beta_2}\Rightarrow 
\|\boldsymbol{v}_t\|_{\infty}\leq G_{\infty}^2(1-\beta_2) 
\Rightarrow 
\|\tilde{\boldsymbol{w}}_t\|_{\infty} \hspace{-0.7mm}\leq\hspace{-0.7mm} \sqrt{G_{\infty}^2+\epsilon} \Rightarrow 
\|\tilde{\boldsymbol{w}}_t\|_{\infty} \hspace{-0.7mm}\leq\hspace{-0.7mm} G_{\infty}+\sqrt{\epsilon}  \\
\sum_{t=1}^{T-1}(1/\eta_{t+1}-1/\eta_t)=\frac{1}{\eta}\sum_{t=1}^{T-1}(\sqrt{t+1}-\sqrt{t}) \leq \sqrt{T}/\eta
\\
\sum_{t=1}^T \frac{\beta_{1t}}{\eta_t}\leq\frac{\beta_{1}}{\eta}\sum_{t=1}^T \lambda^{t-1}\sqrt{t}\leq \frac{\beta_{1}}{\eta}\sum_{t=1}^T \lambda^{t-1}t \leq \frac{\beta_1}{\eta(1-\lambda)^2}. \end{array}\right.}
\end{align}

Next we consider the quantity $\sum_{t=1}^T\eta_t \|\tilde{\boldsymbol{w}}_{t+1}^{-1/2}\boldsymbol{m}_t\|_2^2 $ in (\ref{equ:convex2}), the upper bound of which can be derived in the same manner as Equ.~(4) in the appendix of \cite{Zhuang20Adabelief}: 
\begin{lemma} Let ${\boldsymbol{g}_{1:T}^2[i]}=((\boldsymbol{g}_1[i])^2,\ldots, (\boldsymbol{g}_T[i])^2)\in \mathbb{R}^{T}$. Under the three assumptions given in the theorem, we have 
\begin{align}
    \sum_{t=1}^T\eta_t\|\tilde{\boldsymbol{w}}_{t+1}^{-1/2}\boldsymbol{m}_t\|_2^2  \leq \frac{\eta\sqrt{1+\log T}}{\sqrt{c}(1-\beta_1)^2}\|(\boldsymbol{g}_{1:T}^2[i])\|_2. 
    \label{equ:convex3}
\end{align}
\end{lemma}

Finally, plugging (\ref{equ:convex3}) into (\ref{equ:convex2}) produces the upper-bound regret in the theorem. The proof is complete.   

\end{proof}

\vspace{-2mm}
\section{Parameter-setups for training different DNN models}
\label{appendix:parameterSetups}

\begin{table}[h!]
\caption{\small Parameter-setups  for training ResNet18 over ImageNet. $wd$ in the table refers to the weight-decay parameter.  \vspace{-0mm} } 
\label{tab:setup_imagenet}
\centering
\begin{tabular}{|c|c|c|}
\hline
\footnotesize{optimizer} & \footnotesize{fixed parameters} & \footnotesize{searched parameters}  \\ \hline
\hspace{-4.5mm} 
\footnotesize{AdaBelief} \hspace{-4.5mm} & \hspace{-4.5mm} 
\footnotesize{$\begin{array}{l}(\eta_0,\beta_1, \beta_2, wd)\\
=(0.001, 0.9, 0.999, 10^{-2})\end{array}$} \hspace{-4.5mm} & \hspace{-2.5mm} 
\footnotesize{$\epsilon\in\{10^{-8}, 10^{-9}, 10^{-10}\}$}
\\
\hline
\hspace{-4.5mm} 
\footnotesize{$\begin{array}{c}\textrm{Adam} \\\end{array}$} \hspace{-4.5mm}  &  \hspace{-0.5mm} \footnotesize{$\begin{array}{l}(\eta_0,\beta_1, \beta_2)\\
=(0.001, 0.9, 0.999)\end{array}$}  \hspace{-0.5mm}  \hspace{-0.5mm}  & \footnotesize{ $\epsilon\in \{10^{-3}, 10^{-8}\}$\quad $wd\in\{10^{-2}, 10^{-5}\}$  } \hspace{-0.5mm}
\\
\hline
\hspace{-4.5mm} 
\footnotesize{$\begin{array}{c}\textrm{SET-Adam} \\\end{array}$} \hspace{-4.5mm}  &  \hspace{-0.5mm} \footnotesize{$\begin{array}{l}(\eta_0,\beta_1, \beta_2, \epsilon, \tau,wd)\\
=(0.001, 0.9, 0.999, 10^{-6}, 0.5, 10^{-2})\end{array}$} \hspace{-0.5mm} & \hspace{-4.5mm}
\\
\hline
\end{tabular}
\vspace{-1mm}
\end{table}

%As stated at the beginning of Section-\textbf{Experiments}, the open source ``https://github.com/jadore801120/attention-is-all-you-need-pytorch" was adopted for training a Transformer. All the remaining tasks rely on the second open source 
%``https://github.com/juntang-zhuang/Adabelief-Optimizer", which is in fact for evaluating AdaBelief of \cite{Zhuang20Adabelief}. The four tables below specify the selection of parameters for each optimizer in four DNN training tasks.  

%\textcolor{blue}{In addition, we also evaluated the AdaBound method from \cite{Luo19AdaBound}. The tables below also include the parameter setups for AdaBound. The experimental results for AdaBound are presented in Appendix~G. }  

\begin{table*}[h!]
\caption{\small Parameter-setups  for training a Transformer. The weight decay for AdamW was set to $5e-4$ while the weight decay for all other algorithms was set to 0.0.  \vspace{-0mm} } 
\label{tab:setup_transformer}
\centering
\begin{tabular}{|c|c|c|}
\hline
\footnotesize{optimizer} & \footnotesize{fixed parameters} & \footnotesize{searched parameters}  \\ \hline
\footnotesize{AdaBound} & \footnotesize{$(\eta_0,\beta_1,\beta_2,\gamma)=(0.001,0.9,0.98,0.001)$} & \footnotesize{$\begin{array}{c}\epsilon\in\{1e-6, 1e-7, \ldots,  1e-12\} \\ 
\textrm{final}\_\textrm{stepsize}\in \{0.1, 0.01,0.001\}\end{array}$}
\\
\hline
\footnotesize{Yogi} & \footnotesize{$(\eta_0,\beta_1,\beta_2)=(0.001,0.9,0.98)$} & \footnotesize{$\epsilon\in\{1e-2, 1e-3, \ldots,  1e-8\}$}
\\
\hline
\footnotesize{SGD} & \footnotesize{momentum=0.9} &
\footnotesize{$\eta_0\in\{1.0, 0.1, 0.01, 0.001\}$ }
\\
\hline
\footnotesize{RAdam} & 
\footnotesize{$(\eta_0,\beta_1, \beta_2)=(0.001, 0.9, 0.98)$} & \footnotesize{$\epsilon\in\{1e-6, 1e-7, \ldots,  1e-12\}$}
\\
\hline
\footnotesize{MSVAG} & \footnotesize{$(\eta_0,\beta_1,\beta_2)=(0.001,0.9,0.98)$} & \footnotesize{$\epsilon\in\{1e-6, 1e-7, \ldots,  1e-12\}$}
\\
\hline
\footnotesize{Fromage} & & \footnotesize{$\eta_0\in\{0.1, 0.01, 0.001, 0.0001\}$}
\\
\hline
\footnotesize{Adam} & 
\footnotesize{$(\eta_0,\beta_1, \beta_2)=(0.001, 0.9, 0.98)$} & \footnotesize{$\epsilon\in\{1e-6, 1e-7, \ldots,  1e-12\}$}
\\
\hline
\footnotesize{AdamW} & 
\footnotesize{$(\eta_0,\beta_1, \beta_2)=(0.001, 0.9, 0.98)$} & \footnotesize{$\epsilon\in\{1e-6, 1e-7, \ldots,  1e-12\}$}
\\
\hline
\footnotesize{AdaBelief} & 
\footnotesize{$(\eta_0,\beta_1, \beta_2)=(0.001, 0.9, 0.98)$} & \footnotesize{$\epsilon\in\{1e-8, 1e-10, \ldots,  1e-16\}$}
\\
\hline
\hline
\footnotesize{SET-Adam(\textbf{our})} & 
\footnotesize{$(\eta_0,\beta_1, \beta_2, \epsilon, \tau)=(0.001, 0.9, 0.98, 1e-15, 0.5)$} & 
\\
\hline
\end{tabular}
\vspace{-1mm}
\end{table*}\

\begin{table*}[h!]
\caption{\small Parameter-setups for training LSTMs. The weight decay for every algorithm was set to $1.2e-6$.   \vspace{-0mm} } 
\label{tab:setup_LSTM}
\centering
\begin{tabular}{|c|c|c|}
\hline
\footnotesize{optimizer} & \footnotesize{fixed parameters} & \footnotesize{searched parameters}  \\ \hline
\\ \hline
\footnotesize{AdaBound} & \footnotesize{$(\beta_1, \beta_2,\gamma)=(0.9, 0.999,0.001)$} & \footnotesize{$\begin{array}{c} 
\eta_0 \in \{0.01, 0.001\} \\
\epsilon\in\{1e-6, 1e-8, \ldots, 1e-12\} \\
\textrm{final}\_\textrm{stepsize}\in \{0.1, 3, 30\}\end{array}$} \\
\hline
\footnotesize{Yogi} & \footnotesize{$(\beta_1, \beta_2)=(0.9, 0.999)$} & \footnotesize{$\begin{array}{c}\eta_0\in \{0.01, 0.001\} \\  \epsilon\in\{1e-2, 1e-3, 1e-4, 1e-5\} \end{array}$}
\\
\hline
\footnotesize{SGD} & \footnotesize{momentum=0.9} &
\footnotesize{$\eta_0\in\{30, 3, 1, 0.1\}$ }
\\
\hline
\footnotesize{RAdam} & 
\footnotesize{$(\beta_1, \beta_2)=(0.9, 0.999)$} & \footnotesize{$\begin{array}{c}\eta_0\in \{0.01, 0.001\} \\  \epsilon\in\{1e-6, 1e-8, 1e-10,  1e-12\} \end{array}$}
\\
\hline
\footnotesize{MSVAG} & \footnotesize{$(\beta_1, \beta_2)=(0.9, 0.999)$} & \footnotesize{$\begin{array}{c}\eta_0\in \{30, 1, 0.01, 0.001\} \\  \epsilon\in\{1e-6, 1e-8, 1e-10, 1e-12\} \end{array}$}
\\
\hline
\footnotesize{Fromage} & & \footnotesize{$\eta_0\in\{0.1, 0.01, 0.001\}$}
\\
\hline
\footnotesize{Adam} & 
\footnotesize{$(\beta_1, \beta_2)=(0.9, 0.999)$} & \footnotesize{$\begin{array}{c}\eta_0\in \{0.01, 0.001\} \\  \epsilon\in\{1e-6, 1e-8, 1e-10,  1e-12\} \end{array}$}
\\
\hline
\footnotesize{AdamW} & 
\footnotesize{$(\beta_1, \beta_2)=(0.9, 0.999)$} & \footnotesize{$\begin{array}{c}\eta_0\in \{0.01, 0.001\} \\  \epsilon\in\{1e-6, 1e-8, 1e-10,  1e-12\} \end{array}$}
\\
\hline
\footnotesize{AdaBelief} & 
\footnotesize{$(\beta_1, \beta_2)=(0.9, 0.999)$} & \footnotesize{$\begin{array}{c}\eta_0\in \{0.01, 0.001\} \\  \epsilon\in\{1e-8, 1e-10, \ldots,  1e-16\} \end{array}$}
\\
\hline
\footnotesize{$\begin{array}{c}\textrm{SET-Adam} (\textrm{\textbf{our}}) \\ \textrm{[for layer 1 and 2]}\end{array}$} & 
\footnotesize{$(\eta_0,\beta_1, \beta_2, \epsilon, \tau)=(0.001, 0.9, 0.999, 1e-13, 0.5)$} & 
\\
\hline
\footnotesize{$\begin{array}{c}\textrm{SET-Adam} (\textrm{\textbf{our}}) \\ \textrm{[for layer 3]}\end{array}$} & 
\footnotesize{$(\eta_0,\beta_1, \beta_2, \epsilon, \tau)=(0.001, 0.9, 0.999, 1e-11, 0.5)$} & 
\\
\hline
\end{tabular}
\vspace{-1mm}
\end{table*}

\begin{table*}[h!]
\caption{\small Parameter-setups  for training WGAN-GP.  The weight decay for AdamW was set to $5e-4$ while the weight decay for all other algorithms was set to 0.0.  \vspace{-0mm} } 
\label{tab:setup_WGANGP}
\centering
\begin{tabular}{|c|c|c|}
\hline
\footnotesize{optimizer} & \footnotesize{fixed parameters} & \footnotesize{searched parameters}  \\ \hline
\footnotesize{AdaBound} & \footnotesize{$(\eta_0,\beta_1, \beta_2,\gamma)=(0.0002,0.5, 0.999,0.001)$} & \footnotesize{$\begin{array}{c}  \epsilon\in\{1e-2, 1e-4, \ldots, 1e-10\} \\
\textrm{final}\_\textrm{stepsize}\in \{0.1, 0.01\}\end{array}$}
\\
\hline
\footnotesize{Yogi} & \footnotesize{$(\eta_0,\beta_1, \beta_2)=(0.0002,0.5, 0.999)$} & \footnotesize{$\begin{array}{c}  \epsilon\in\{1e-2, 1e-3, 1e-4, 1e-5\} \end{array}$}
\\
\hline
\footnotesize{SGD} &  &
\footnotesize{$\begin{array}{c}
\textrm{momentum} = \{0.3, 0.5, 0.9\}
\\\eta_0\in\{0.1, 0.02, 0.002, 0.0002\}\end{array}$ }
\\
\hline
\footnotesize{RAdam} & 
\footnotesize{$(\eta_0,\beta_1, \beta_2)=(0.0002,0.5, 0.999)$} & \footnotesize{$\begin{array}{c} \epsilon\in\{1e-4, 1e-6,\ldots,  1e-14\} \end{array}$}
\\
\hline
\footnotesize{MSVAG} & \footnotesize{$(\beta_1, \beta_2)=(0.5, 0.999)$} & \footnotesize{$\begin{array}{c}\eta_0\in \{0.1, 0.02, 0.002, 0.0002\} \\  \epsilon\in\{1e-4, 1e-6, \ldots, 1e-14\} \end{array}$}
\\
\hline
\footnotesize{Fromage} & & \footnotesize{$\eta_0\in\{0.1, 0.01, 0.001\}$}
\\
\hline
\footnotesize{Adam} & 
\footnotesize{$(\eta_0,\beta_1, \beta_2)=(0.0002,0.5, 0.999)$} & \footnotesize{$\begin{array}{c} \epsilon\in\{1e-4, 1e-6, \ldots,  1e-14\} \end{array}$}
\\
\hline
\footnotesize{AdamW} & 
\footnotesize{$(\eta_0, \beta_1, \beta_2)=(0.0002,0.5, 0.999)$} & \footnotesize{$\begin{array}{c} \epsilon\in\{1e-4, 1e-6, \ldots,  1e-14\} \end{array}$}
\\
\hline
\footnotesize{AdaBelief} & 
\footnotesize{($\eta_0,\beta_1, \beta_2, \epsilon)=(0.0002, 0.5, 0.999, 1e-12)$} & 
\\
\hline
\footnotesize{SET-Adam} (\textbf{our}) & 
\footnotesize{($\eta_0,\beta_1, \beta_2, \epsilon, \tau)=(0.0002, 0.5, 0.999, 1e-11, 0.5)$} & 
\\
\hline
\hline
\end{tabular}
\vspace{-1mm}
\end{table*}

\begin{table*}[h!]
\caption{\small Parameter-setups  for training VGG and ResNet models over CIFAR10 and CIFAR100.  The weight decay for AdamW was set to $0.01$ while the weight decay for all other algorithms was set to $5e-4$.   \vspace{-0mm} } 
\label{tab:setup_VGGResNet}
\centering
\begin{tabular}{|c|c|c|}
\hline
\footnotesize{optimizer} & \footnotesize{fixed-parameters}  & \footnotesize{searched-parameters} \\ \hline
\footnotesize{AdaBound} & \footnotesize{$(\eta_0,\beta_1, \beta_2,\gamma)=(0.001,0.9, 0.999, 0.001)$} & \footnotesize{
$\begin{array}{c}\epsilon\in \{1e-2,1e-3, \ldots, 1e-8\}
\\ \textrm{final}\_\textrm{stepsize} \in\{0.1, 0.01\} \end{array}$ }
\\
\hline
\footnotesize{Yogi} & \footnotesize{$(\eta_0,\beta_1, \beta_2,)=(0.001,0.9, 0.999$)} & \footnotesize{$\epsilon\in \{1e-1,1e-2, 1e-3\}$ }
\\
\hline
\footnotesize{SGD} &  
\footnotesize{$(\textrm{momentum}, \eta_0)=(0.9,0.1)$}
\\
\hline
\footnotesize{RAdam} & 
\footnotesize{$(\eta_0,\beta_1, \beta_2)=(0.001,0.9, 0.999)$} &  \footnotesize{$\epsilon\in \{1e-2, 1e-3, \ldots, 1e-8\}$}
\\
\hline
\footnotesize{MSVAG} & \footnotesize{$(\eta_0,\beta_1, \beta_2)=(0.1,0.9, 0.999)$} & \footnotesize{$\epsilon\in \{1e-2, 1e-2,\ldots, 1e-8\}$}
\\
\hline
\footnotesize{Fromage} & &
\footnotesize{$\eta_0\in \{0.1, 0.01, 0.001\}$}
\\
\hline
\footnotesize{Adam} & 
\footnotesize{$(\eta_0,\beta_1, \beta_2)=(0.001,0.9, 0.999)$} & \footnotesize{$\epsilon\in\{1e-2,1e-3,\ldots, 1e-8\}$}
\\
\hline
\footnotesize{AdamW} & 
\footnotesize{$(\eta_0,\beta_1, \beta_2)=(0.001,0.9, 0.999)$} & \footnotesize{$\epsilon\in\{1e-2,1e-3,\ldots, 1e-8\}$}
\\
\hline
\footnotesize{AdaBelief} & 
\footnotesize{$(\eta_0,\beta_1, \beta_2)=(0.001,0.9, 0.999)$} & \footnotesize{$\epsilon\in\{1e-8, 1e-9\}$}
\\
\hline
\footnotesize{SET-Adam ($\textbf{our}$) }  & 
\footnotesize{$(\eta_0,\beta_1, \beta_2, \epsilon, \tau)=(0.001,0.9, 0.999, 1e-5, 0.5)$} &
\\
\hline
\end{tabular}
\vspace{-1mm}
\end{table*}

%%%%%%%%%%%%%%%%%%%%%%%%%%%%%%%%%%%%%%%%%%%%%%%%%%%%%%%%%%%%%%%%%%%%%%%%%%%%%%%
%%%%%%%%%%%%%%%%%%%%%%%%%%%%%%%%%%%%%%%%%%%%%%%%%%%%%%%%%%%%%%%%%%%%%%%%%%%%%%%

\end{document}